\documentclass{article}

\usepackage[preprint]{neurips_2026}

\usepackage[utf8]{inputenc} 
\usepackage[T1]{fontenc}    
\usepackage{url}            
\usepackage{booktabs}       
\usepackage{amsfonts}       
\usepackage{nicefrac}       
\usepackage{microtype}      
\usepackage{xcolor}         
\usepackage{comment}
\usepackage{amsthm}
\usepackage{xspace}
\title{Formatting Instructions For NeurIPS 2026}

\newcommand{\ourtitle}{Beyond Parameter Aggregation: Semantic Consensus for Federated Fine-Tuning of LLMs}

\input{preamble}

\title{\ourtitle}

\usepackage[capitalize,noabbrev]{cleveref}
\usepackage[textsize=tiny]{todonotes}
\usepackage{soul}
\usepackage{multirow}
\usepackage[dvipsnames]{xcolor}

\author{%
  Amr Abourayya\\
  Lamarr Institute for ML and AI, Technical University Dortmund\\
Institute for AI in medicine (IKIM),University Hospital Essen\\
  \texttt{amr.abourayya@tu-dortmund.de} \\
  \And
  Jens Kleesiek\\
  Institute for AI in medicine (IKIM)\\
  University Hospital Essen\\
  \And
  Michael Kamp\\
   Lamarr Institute for ML and AI, Technical University Dortmund\\
  Institute for AI in medicine (IKIM), University Hospital Essen\\
  \texttt{}
}

\begin{document}



\newcommand{\TODO}[1]{\textcolor{red}{TODO: #1}}

\newcommand{\defemph}[1]{\emph{#1}}

\newcommand*\Let[2]{\State #1 $\gets$ #2}

\newcommand{\fedcofit}{FedCoFiT\xspace}
\newcommand{\E}{\mathbb{E}}


\theoremstyle{plain}
\newtheorem{theorem}{Theorem}
\newtheorem{lemma}{Lemma}

\theoremstyle{definition}
\newtheorem{assumption}{Assumption}
\newtheorem{definition}{Definition}
\newcommand{\R}{\mathbb{R}}
\newcommand{\KL}{\mathrm{KL}}
\newcommand{\TV}{\mathrm{TV}}

\newcommand{\algo}{\mathcal{A}}
\newcommand{\risk}{\varepsilon}
\newcommand{\model}{h}
\newcommand{\modelspace}{\mathcal{H}}
\newcommand{\aggmodel}{\overline{\model}}
\newcommand{\loss}{\ell}
\newcommand{\samplesize}{n}
\newcommand{\radonpoint}{\mathfrak{r}}


\newcommand{\Exp}[2]{\mathop{{}\mathbb{E}_{#1}} \Big[ #2 \Big] }
\newcommand{\Prob}[2]{\mathop{{}\mathbb{P}_{#1}} \left ( #2 \right ) }
\newcommand{\prob}[1]{\mathop{{}\mathbb{P}} \left(  #1 \right ) }

\newcommand{\Dcal}{\mathcal{D}}
\newcommand{\RR}{\mathbb{R}}
\newcommand{\NN}{\mathbb{N}}

\newcommand{\Xcal}{\mathcal{X}}
\newcommand{\Ycal}{\mathcal{Y}}
\newcommand{\bigo}{\mathcal{O}}

\newcommand{\errorband}[5][]{ 
    \addplot+[draw=none, stack plots=y, forget plot] table [
        x expr=\thisrowno{#3},
        y index={#4}
    ] {#2};

    \addplot+[forget plot, draw=none, fill=gray!40, stack plots=y, #1] table [
        x expr=\thisrowno{#3},
        y expr=\thisrowno{#5} - \thisrowno{#4}
    ] {#2} \closedcycle;

    \addplot+[forget plot, stack plots=y,draw=none] table [x index={#4}, y expr=-(\thisrowno{#5})] {#2};
}

\maketitle
\vspace{-0.5cm}
\begin{abstract}
Federated fine-tuning of large language models is commonly formulated as a parameter aggregation problem. However, even parameter-efficient methods require transmitting large collections of trainable weights, assume aligned architectures, and rely on white-box access to model parameters. As model sizes continue to grow and deployments become increasingly heterogeneous, these assumptions become progressively misaligned with practical constraints.
We consider an alternative formulation in which collaboration is mediated through model behavior rather than parameters. Clients fine-tune local models on private data and exchange generated outputs on a shared, public prompt set. The server maps these outputs into a semantic representation space, forms a per-prompt semantic consensus, and returns pseudo-labels for further local fine-tuning.
This formulation fundamentally changes the communication scaling of federated LLM fine-tuning. The amount of information exchanged depends only on the public prompt budget and the size of the communicated behaviors, independent of model size. As a consequence, the protocol naturally accommodates heterogeneous architectures and applies directly to open-ended text generation.
We present a theoretical analysis and empirical results demonstrating that this approach can match strong federated fine-tuning baselines while substantially reducing communication by orders of magnitude (e.g., analytically by a factor of $1006$ for Llama3.1-405B),  as well as reductions in runtime and energy consumption. These results suggest that, for generative foundation models, behavior-level consensus provides a more appropriate abstraction for federated adaptation than parameter aggregation.
\end{abstract}

\section{Introduction}
\label{sec:introduction}

Large language models are increasingly being adapted to specific sensitive domains, including healthcare, finance, and enterprise workflows. In many practical deployments, the data required for effective fine-tuning is distributed across organizations and devices and cannot be centralized due to privacy, legal, or competitive constraints. Federated learning provides a natural framework for collaborative adaptation in such settings, enabling participants to improve local models without exposing sensitive raw data. Yet for modern LLMs, the dominant practical constraint is often not local compute but \textbf{what can be communicated.}

Most federated learning protocols are parameter-centric, which means that clients transmit and aggregate weights, gradients, or adapter updates. For LLMs, this formulation faces three fundamental difficulties. First, communication cost becomes a dominant bottleneck. Even parameter-efficient fine-tuning approaches require repeatedly transmitting large collections of trainable weights, which are increasingly constrained by realistic bandwidth limits. Second, clients may differ in LoRA rank, adapted layers, tokenizers, or even base architectures, making parameter aggregation difficult or ill-defined. Third, white-box access to model parameters is often unavailable in practice. 

These difficulties reflect a more fundamental misalignment between parameter aggregation and the deployment and use of modern LLMs. In practice, what downstream systems consume from an LLM is not its internal parameters but its behaviour. For example, responses to prompts, preference judgments, tool calls, structured outputs, safety decisions, or intermediate reasoning artifacts emitted by the interface. This motivates a more general question: \textbf{Can federated collaboration be mediated through behaviour rather than parameters?}

At a high level, behaviour sharing treats each client model as a black-box client that can be queried on shared public stimuli and that can return compact, shareable summaries of its behaviour. The shared object need not be restricted to predictions in the classical supervised sense, depending on the application. It could be ranked preferences (for alignment), structured API calls (for tool-augmented workflows), extracted entities or plans (for enterprise pipelines), or free-form generation (for instruction following). If collaboration is expressed in a common \emph{semantic} space over such behaviours, then communication can be decoupled from model size and remain meaningful even when clients are heterogeneous.

We call this general approach \textbf{Semantic Consensus}: a framework that builds on top of the federated co-training approach \citep{abourayya2025little} in which clients exchange model behaviour on a shared public prompt set and a server aggregates these behaviours to construct consensus targets that guide further local fine-tuning. For LLMs, behaviour is naturally captured by responses generated to a set of shared public prompts. By embedding these responses into a semantic space, consensus can be formed via geometric operations (e.g., clustering) that identify semantically consistent groups and select representative behaviours as supervision. We provide a concrete instantiation of this principle in Sec.~\ref{sec:method}.
This design directly addresses the main friction points of the federated LLM adaptation. Communication becomes independent of model size, since clients transmit only generated text responses and receive consensus targets. Because aggregation occurs in the behaviour space, the framework naturally accommodates heterogeneity in architectures, tokenizers, and LoRA configurations. Finally, Semantic Consensus requires only forward access to model behaviour, making it applicable to open-ended generation without task-specific heads.

In summary, our contributions are as follows.
\begin{itemize}
    \item We introduce \textbf{Semantic Consensus}, a general federated co-training principle that replaces parameter aggregation with consensus over client behaviour on shared public inputs, yielding communication independent of model size.
    \item We propose \textbf{FedCoFiT}, an instantiation of Semantic Consensus for open-ended text generation that builds consensus using semantic embeddings, similarity metrics, and clustering to obtain high-quality pseudo-labels for further local fine-tuning.
    \item We provide a theoretical analysis explaining why behaviour-sharing yields substantially improved communication scaling relative to adapter aggregation, and we show that \fedcofit is compatible with record-level differential privacy.
    \item We empirically evaluate \fedcofit across multiple benchmarks, showing that it can match or approach strong federated fine-tuning baselines while significantly reducing communication and improving practical efficiency (runtime and energy consumption).
\end{itemize}

\section{Related Work}
\label{sec:related_work}
Traditional federated learning approaches rely on iterative model averaging (e.g., FedAvg \cite{mcmahan2017communication}), requiring clients to upload full model updates constantly. Although effective for smaller architectures, it becomes impractical for large language models due to the high cost of communication. A promising solution to this bottleneck is the adoption of parameter-efficient fine-tuning (PEFT) within the Federated Learning framework.

Parameter-efficient federated fine-tuning of LLMs such as LoRA \cite{hu2022lora} is widely adopted to reduce the number of trainable parameters, and several works adapt LoRA to federated fine-tuning by aggregating LoRA updates at the server. A key challenge is that naively averaging LoRA adapters can be mathematically inconsistent, introducing aggregation noise and degrading performance when client configurations differ. For example, FedIT \cite{zhang2024towards} integrates LoRA directly with FedAvg by having the server independently average the local LoRA adapters \textbf{A} and \textbf{B} received from clients. While this reduces communication costs, it cannot aggregate local LoRAs with heterogeneous ranks and introduces aggregation noise as the matrix multiplication required to reconstruct the update is non-linear. Specifically, the product of the averaged matrices is not equal to the sum of the individual products, leading to mathematical errors in the global model.

FLoRA addresses this issue via a stacking-based aggregation method that supports heterogeneous LoRA adapters and improves the stability of federated LoRA fine-tuning \cite{wang2024flora}. FlexLoRA synthesizes and redistributes LoRA weights using SVD to exploit heterogeneous client budgets \cite{bai2024federated}, while HetLoRA (and related heterogeneous-LoRA methods) allows different ranks across devices and aggregates while accounting for differences in data and sparsity patterns \cite{cho2024heterogeneous}. More recent studies improved initialization, aggregation, and robustness for federated LoRA tuning \cite{bian2025lora,koo2025towards}. These approaches are strong baselines when clients share a compatible parameter space and have white-box access to adapters. However, they fundamentally rely on transmitting trainable parameters (or structured adapter weights), which remains costly at scale and becomes inapplicable when clients use different architectures, tokenizers, or only have black-box access. More recent work continues to refine parameter-space aggregation for federated LoRA.  Fed-PLoRA~\cite{zhang2026heterogeneous} introduces parallel one-rank LoRA modules and a Select-N-Fold strategy to reduce initialization and aggregation noise under heterogeneous client ranks, while SFed LoRA~\cite{huang2026stabilized} studies the interaction between LoRA rank, client count, and aggregation stability through a rank-dependent scaling factor. These methods further improve heterogeneous LoRA aggregation, but they remain parameter-centric as clients must still communicate adapter-side information and the server must reason about compatible LoRA structures.

An alternative to parameter aggregation is to avoid sharing weights altogether, e.g., via distillation such as \citep{bistritz2020distributed} or black-box control. Distillation-based FL can be viewed as a special case of behavior sharing under a fixed label space, where the shared behavior is a soft predictive distribution over classes. FedID~\cite{ma2023fedid}
is a direct application of this distillation paradigm, but it applies only to classification with fixed label spaces and does not extend to open-ended generation, where token-level supervision and an unconstrained output space make aggregation more complex and bandwidth-intensive. Additionally, it assumes server-held labels and iterative feedback. When parameters are inaccessible, black-box FL instead tunes prompts via gradient-free updates (e.g., FedBPT~\cite{sun2023fedbpt}), but prompt tuning is often less expressive than weight adaptation and can be query-intensive.

\section{Method}
\label{sec:method}

\begin{figure*}[ht]
    \centering
    \vspace{-1.3cm}
    \includegraphics[width=\textwidth]{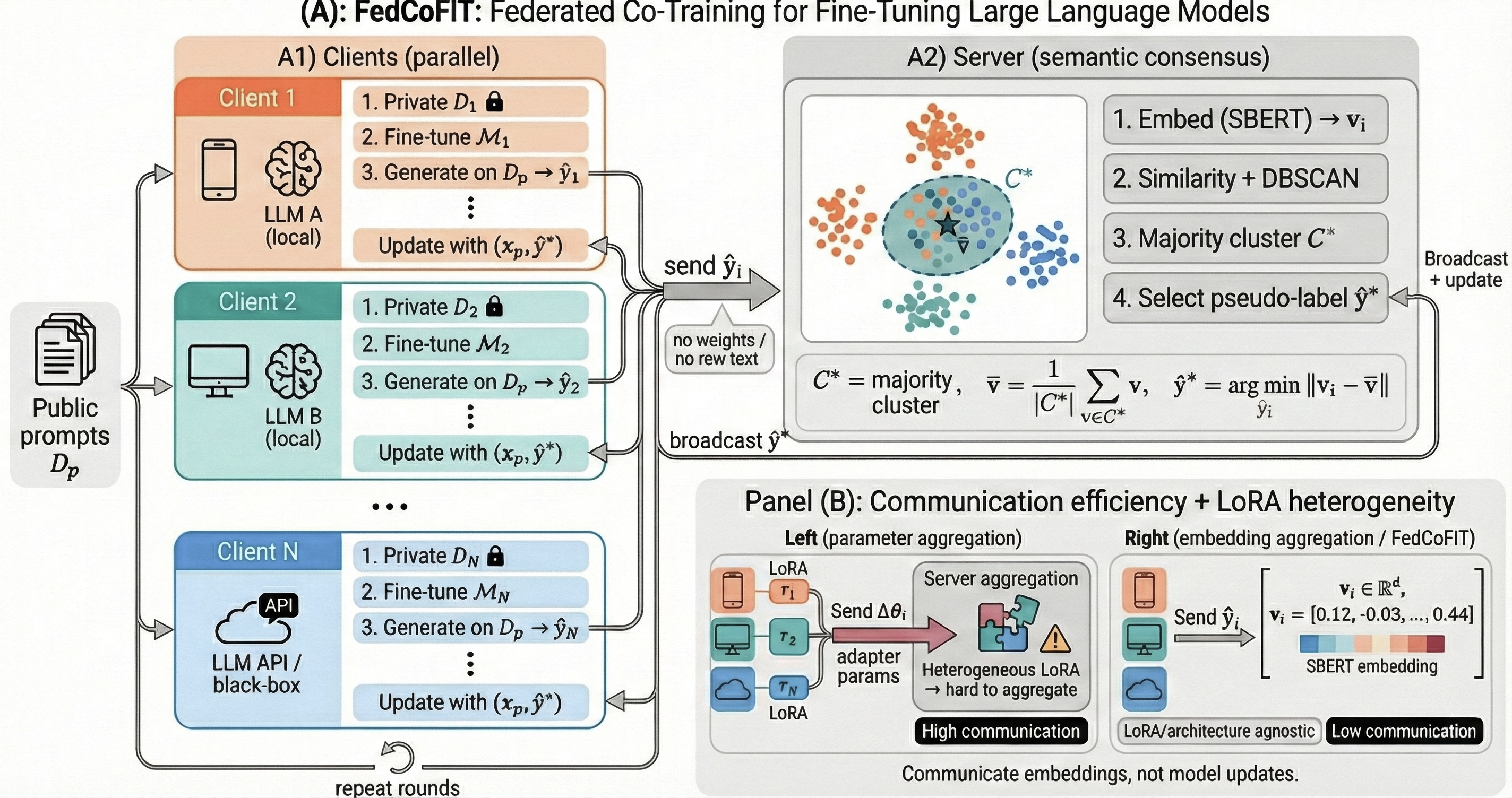}
    \caption{
    Overview of \fedcofit.
    Clients fine-tune locally, generate responses to shared public prompts, and send only these responses to the server.
    The server embeds and clusters the responses to select consensus pseudo-labels, which are broadcast back for further local training.
    Unlike parameter aggregation, \fedcofit aggregates behavior and is compatible with heterogeneous clients.
    }
    \label{fig:fedcofit}
\end{figure*}

We consider $K$ clients indexed by $i\in\{1,\dots,K\}$. 
Client $i$ has a private dataset $D_i=\{z_{i,k}\}_{k=1}^{n_i}$ drawn from a client-specific distribution $\mathcal{P}_i$ and a local language model $\mathcal{M}_i$. 
Clients may differ in architecture, tokenizer, adaptation module, LoRA rank, or access level. 
They collaborate without sharing raw data, gradients, optimizer states, model parameters, or adapter weights. 
All clients also have access to a public prompt set $D_{\mathrm{pub}}=\{x_j\}_{j=1}^{M}$ drawn from a public prompt distribution $\mathcal{P}_{\mathrm{pub}}$, containing inputs only. 
Let $\mathcal{B}_i(x)\in\mathcal{Y}$ denote the observable behavior of client $i$ on input $x$, such as a generated response, preference ranking, tool call, or structured output. 
Our goal is to use $D_{\mathrm{pub}}$ to transfer behavior across clients while keeping communication independent of model size, avoiding parameter- or adapter-alignment assumptions, and preserving the federated constraint that private data remain local.


\textbf{Semantic Consensus} is a behavior-level aggregation principle for federated fine-tuning. 
Instead of averaging parameters, clients expose their behavior on public prompts, and the server aggregates these behaviors in a shared semantic space. 
Let $g:\mathcal{Y}\rightarrow\mathbb{R}^{d_e}$ be a semantic encoder and let $\mathrm{dist}(\cdot,\cdot)$ be a distance function in the embedding space. 
For each public prompt $x_j$, client $i$ produces a behavior $b_{i,j}=\mathcal{B}_i(x_j)$ and sends it to the server. 
The server computes $e_{i,j}=g(b_{i,j})$ and aggregates $\{e_{i,j}\}_{i=1}^{K}$ to form a consensus target $t_j$. 
This yields a pseudo-labeled public set $\tilde{D}_{\mathrm{pub}}=\{(x_j,t_j)\}_{j=1}^{M}$, which is returned to clients for local fine-tuning. 
Because aggregation occurs over behaviors, Semantic Consensus does not require clients to share architectures, tokenizers, adapter shapes, or white-box model access.

\subsection{\fedcofit: Semantic Consensus for Generative LLM Fine-Tuning}
\label{sec:fedcofit}

\SetKwFor{local}{Client-side}{do}{}
\SetKwFor{coord}{Server-side}{do}{}

\begin{wrapfigure}{r}{0.52\textwidth}
\vspace{-1.3em}
\hspace{1.3em}
\begin{minipage}{0.52\textwidth}
\begin{algorithm}[H]
    \small
    \caption{\fedcofit via Semantic Consensus}
    \label{alg:fedcofit}
    \KwIn{
    clients $\{1,\dots,K\}$, private datasets $\{D_i\}$; 
    public prompts $D_{\mathrm{pub}}=\{x_j\}_{j=1}^{M}$; 
    encoder $g$; distance $\mathrm{dist}$; clustering rule $\mathrm{Cluster}$
    }
    \KwOut{updated models $\mathcal{M}_1,\dots,\mathcal{M}_K$}

    initialize $\mathcal{M}_1,\dots,\mathcal{M}_K$\;

    \Repeat{convergence}{
        \ForEach{client $i\in\{1,\dots,K\}$ in parallel}{
            update $\mathcal{M}_i$ on $D_i$\;
            generate $\hat{y}_{i,j}\leftarrow \mathcal{M}_i(x_j)$ for all $x_j\in D_{\mathrm{pub}}$\;
            send $\{\hat{y}_{i,j}\}_{j=1}^{M}$ to server\;
        }

        \ForEach{$x_j\in D_{\mathrm{pub}}$}{
            compute $e_{i,j}=g(\hat{y}_{i,j})/\|g(\hat{y}_{i,j})\|_2$, $i\in\{1,\dots,K\}$\;
            cluster $\{e_{i,j}\}_{i=1}^{K}$ using $\mathrm{Cluster}$ and $\mathrm{dist}$\;
            select consensus cluster $C_j^\star$\;
            compute centroid $\bar{e}_j$ of $C_j^\star$\;
            choose $i^\star=\arg\min_{i\in C_j^\star}\mathrm{dist}(e_{i,j},\bar{e}_j)$\;
            set $y_j^\star=\hat{y}_{i^\star,j}$\;
        }

        broadcast $\tilde{D}_{\mathrm{pub}}=\{(x_j,y_j^\star)\}_{j=1}^{M}$\;

        \ForEach{client $i\in\{1,\dots,K\}$ in parallel}{
            update $\mathcal{M}_i$ on $D_i\cup\tilde{D}_{\mathrm{pub}}$\;
        }
    }
\end{algorithm}
\end{minipage}
\vspace{-1.0em}
\end{wrapfigure}

We instantiate Semantic Consensus for open-ended generation as \textbf{\fedcofit}. 
For each public prompt $x_j$, client $i$ generates a response $\hat{y}_{i,j}\leftarrow\mathcal{M}_i(x_j)$. 
The server embeds each response with a sentence encoder, normalizes the embedding, clusters the responses by semantic similarity, and selects a representative response as the pseudo-label. 
For each response $\hat{y}_{i,j}$, the server computes $v_{i,j}=g(\hat{y}_{i,j})\in\mathbb{R}^{d_e}$ and normalizes it as $e_{i,j}=v_{i,j}/\|v_{i,j}\|_2$. 
All clustering and representative-selection steps are performed on the normalized embeddings. 
We use cosine distance, $\mathrm{dist}(e,e')=1-e^\top e'$, so semantically similar responses are close even when they differ in surface form. 
The encoder choice, decoding parameters, and clustering hyperparameters are fixed within each run and reported in the experimental setup.

For each prompt $x_j$, the server clusters the embeddings $\{e_{i,j}\}_{i=1}^{K}$ using a density-based clustering rule. 
Density-based clustering is useful because the number of semantic response modes is unknown and isolated responses can be treated as outliers. 
The server selects the largest non-outlier cluster as the consensus cluster $C_j^\star$. 
If multiple clusters have the same size, the server selects the cluster with the smallest average pairwise cosine distance, with any remaining tie broken by the smallest client index. 
If all responses are marked as outliers, the server uses all responses as the candidate set by setting $C_j^\star=\{1,\dots,K\}$. 
Given $C_j^\star$, the server computes its normalized centroid $\bar{e}_j=\sum_{i\in C_j^\star}e_{i,j}/\|\sum_{i\in C_j^\star}e_{i,j}\|_2$. 
The pseudo-label is the original response closest to this centroid: $i^\star=\arg\min_{i\in C_j^\star}\mathrm{dist}(e_{i,j},\bar{e}_j)$ and $y_j^\star=\hat{y}_{i^\star,j}$. 
If multiple responses are equally close to the centroid, ties are broken by the shortest response length and then by the smallest client index. 
This keeps the pseudo-label as a real client-generated response rather than requiring the server to synthesize a new target. The server constructs $\tilde{D}_{\mathrm{pub}}=\{(x_j,y_j^\star)\}_{j=1}^{M}$ and broadcasts it to all clients. 
Each client then continues fine-tuning on $D_i\cup\tilde{D}_{\mathrm{pub}}$ using the standard causal language-modeling loss. 
Private examples and consensus pseudo-labeled examples receive equal loss weight. 
The local objective is
\[
\mathcal{L}_i(\theta_i)
=
\mathbb{E}_{(x,y)\sim D_i}
\big[
\ell_{\mathrm{CE}}(\mathcal{M}_{\theta_i}(x),y)
\big]
+
\mathbb{E}_{(x,y^\star)\sim\tilde{D}_{\mathrm{pub}}}
\big[
\ell_{\mathrm{CE}}(\mathcal{M}_{\theta_i}(x),y^\star)
\big],
\]
where $\ell_{\mathrm{CE}}$ is applied to response tokens. 
In the convergence analysis, this corresponds to the normalized objective weights $\alpha=\beta=1/2$. 
\fedcofit performs aggregation in semantic behavior space rather than parameter space. 
The majority-cluster rule favors responses supported by multiple clients, while centroid-based representative selection keeps the pseudo-label as a natural response generated by an actual client. 
Since clients communicate only public-prompt behaviors and receive pseudo-labels, the protocol remains communication-efficient, architecture-agnostic, and compatible with heterogeneous LLMs.

\section{Theory}
\label{theory}

\subsection{Communication Complexity}
\label{sec:complexity}

Semantic Consensus avoids parameter exchange and instead communicates model behavior on a shared public prompt set. Let $M=|D_{\mathrm{pub}}|$ be the number of public prompts per round, $K$ the number of clients, and $\bar{\ell}$ the average response length in tokens. In each round, every client uploads one text response per public prompt, so the upload scales as $KM\bar{\ell}$. The server computes semantic embeddings and performs clustering and representative selection locally. It then broadcasts one pseudo-label per public prompt to all clients, so the download also scales as $KM\bar{\ell}$. Let $c_t$ be the average number of bytes per token. The total per-round communication, upload plus download and aggregated over all clients, is
$C_{\mathrm{SC}}=\Theta\!\big(KM\bar{\ell}c_t\big)$,
which is independent of model size and does not require aligned architectures or adapter shapes.

\paragraph{Parameter sharing baselines.}
In contrast, parameter-centric FL methods communicate model updates. For parameter subsampling, client $i$ transmits a fraction $\rho\in(0,1]$ of an update vector with $P_i$ parameters. If $c_{\mathrm{param}}$ denotes bytes per communicated scalar, the per-round total communication scales as
$C_{\mathrm{subsample}}=\Theta\!\big(K\rho P_i c_{\mathrm{param}}\big)$,
up to comparable download cost for broadcasting aggregated updates.

\paragraph{LoRA aggregation.}
Suppose LoRA is applied to $L_{\mathrm{layers}}$ layers at client $i$, with rank $r_i$, and layer $\ell$ has dimensions $(d^{\mathrm{in}}_\ell,d^{\mathrm{out}}_\ell)$. The LoRA parameter count is $\sum_{\ell=1}^{L_{\mathrm{layers}}} r_i(d^{\mathrm{in}}_\ell+d^{\mathrm{out}}_\ell)$, so the per-round total communication scales as
$C_{\mathrm{LoRA}}
=
\Theta\!\left(
Kc_{\mathrm{param}}
\sum_{\ell=1}^{L_{\mathrm{layers}}}
r_i(d^{\mathrm{in}}_\ell+d^{\mathrm{out}}_\ell)
\right)$,
up to method-specific overhead. This cost scales with model geometry and LoRA rank and assumes aligned adapter shapes. Semantic Consensus instead scales with $(K,M,\bar{\ell})$ and remains valid under LoRA, architecture, tokenizer, and access heterogeneity.

\paragraph{Example: Llama 3.1 405B scale.}
Consider $K=10$ clients, $M=1024$ prompts per round, average response length $\bar{\ell}=128$, and $c_t\approx 2$ bytes/token. 
Semantic Consensus communicates only public-prompt responses and pseudo-labels, so its total upload-plus-download cost is
\[
C_{\mathrm{SC}}
\approx
2KM\bar{\ell}c_t
=
2\cdot10\cdot1024\cdot128\cdot2
\approx
5.2~\mathrm{MiB},
\]
where the leading factor $2$ accounts for upload and download. 
Now consider a Llama~3.1--405B-scale backbone with $L_{\mathrm{layers}}=126$, hidden size $d_{\mathrm{model}}=16{,}384$, FFN size $d_{\mathrm{ff}}=53{,}248$, FP16 communication $c_{\mathrm{param}}=2$, and LoRA rank $r=32$. 
If LoRA is applied only to $W_q$ and $W_v$, the upload cost is
$C^{\uparrow}_{\mathrm{LoRA}(q,v)}=c_{\mathrm{param}}L_{\mathrm{layers}}2r(d_{\mathrm{model}}+d_{\mathrm{model}})=528{,}482{,}304$ bytes $\approx504.0~\mathrm{MiB}$, or $1008.0~\mathrm{MiB}$ including download, which is $\approx194\times$ larger than Semantic Consensus. 
If LoRA is applied to all attention projections and the MLP, the upload cost is
$C^{\uparrow}_{\mathrm{LoRA}(\mathrm{attn+MLP})}=c_{\mathrm{param}}L_{\mathrm{layers}}r\{4(d_{\mathrm{model}}+d_{\mathrm{model}})+3(d_{\mathrm{model}}+d_{\mathrm{ff}})\}\approx2{,}741{,}501{,}952$ bytes $\approx2614.5~\mathrm{MiB}$, or $5229.0~\mathrm{MiB}$ including download, which is $\approx1006\times$ larger. 
Thus, for a 405B-scale model, Semantic Consensus reduces total per-round communication by roughly $194\times$--$1006\times$ while remaining independent of model size. Additional communication examples are provided in Appendix.\ref{app:llama2:example}

\subsection{Convergence}
\label{sec:convergence}

FedCoFiT alternates between local updates on private data and local updates on a pseudo-labeled public prompt set. Since pseudo-labels may be noisy and the public prompt distribution may differ from a reference prompt distribution, each client performs SGD on a time-varying surrogate objective whose gradient is a biased approximation of a fixed reference objective gradient. In the non-IID setting, client distributions differ, introducing additional variability through client heterogeneity. We provide a descent-to-stationarity guarantee where the bias is controlled by covariate shift, pseudo-label noise, and finite-sample concentration from using $B_{\mathrm{pub}}$ public examples per step.

Let there be $K$ clients. Client $i$ draws i.i.d. data from $\mathcal{P}_i(x,y)=\mathcal{P}_{i,X}(x)\mathcal{P}_i(y\mid x)$. Define the global mixture $\mathcal{P}(x,y)=K^{-1}\sum_{i=1}^{K}\mathcal{P}_i(x,y)$, with marginal $\mathcal{P}_X$ and conditional $\mathcal{P}(y\mid x)$. The private objectives are $F_i(\theta)=\mathbb{E}_{(x,y)\sim\mathcal{P}_i}[\ell(\theta;x,y)]$ and $F(\theta)=K^{-1}\sum_{i=1}^{K}F_i(\theta)$.

At communication round $q$, public prompts follow $Q_X^{(q)}$. Given prompt $x$, the pseudo-labeling mechanism induces $\widetilde{\mathcal{P}}^{(q)}(\cdot\mid x)$ over pseudo-labels $\tilde{Y}$. Since the public pool is unlabeled, define the reference public objective as $G(\theta)=\mathbb{E}_{x\sim\mathcal{P}_X}\mathbb{E}_{y\sim\mathcal{P}(\cdot\mid x)}[\ell(\theta;x,y)]$. For the equal-weight objective used in \fedcofit, define the round-$q$ surrogate objective as
\begin{equation}
\Phi^{(q)}(\theta)
=
\frac{1}{2}F(\theta)
+
\frac{1}{2}G(\theta;\widetilde{\mathcal{P}}^{(q)},Q_X^{(q)}),
\end{equation}
where $G(\theta;\widetilde{\mathcal{P}}^{(q)},Q_X^{(q)})=\mathbb{E}_{x\sim Q_X^{(q)}}\mathbb{E}_{\tilde{Y}\sim\widetilde{\mathcal{P}}^{(q)}(\cdot\mid x)}[\ell(\theta;x,\tilde{Y})]$. We measure stationarity with respect to $\bar{\Phi}(\theta)=\frac{1}{2}F(\theta)+\frac{1}{2}G(\theta)$.

\begin{assumption}[Smoothness and stochastic gradients]
\label{ass:smooth_sgd_noniid}
Assume: (i) $\bar{\Phi}$ is $\Lambda$-smooth along $\{\theta^t\}_{t=0}^{T}$; and (ii) for each SGD step $t$, $\mathbb{E}[g^t\mid\theta^t,q(t)]=\nabla\Phi^{(q(t))}(\theta^t)$ and
$
\mathbb{E}\!\left[\|g^t-\nabla\Phi^{(q(t))}(\theta^t)\|^2\mid\theta^t,q(t)\right]\le\sigma^2+\frac{1}{4}\zeta^2.
$
Here $\sigma^2$ captures minibatch noise, and $\zeta^2$ captures client heterogeneity.
\end{assumption}

\begin{assumption}[Client heterogeneity]
\label{ass:heterogeneity_noniid}
Along the iterates $\{\theta^t\}_{t=0}^{T}$, assume $K^{-1}\sum_{i=1}^{K}\|\nabla F_i(\theta^t)-\nabla F(\theta^t)\|^2\le\zeta^2$ for all $t\in\{0,\dots,T\}$.
\end{assumption}

\begin{assumption}[Public shift, pseudo-label noise, and finite public sampling]
\label{ass:shift_noise_noniid}
Assume: (i) for each round $q$, $\mathrm{KL}(\mathcal{P}_X\|Q_X^{(q)})\le\tau_q$ and $\tau=\max_q\tau_q$; (ii) for all $(\theta,x,y)$, $\|\nabla_\theta\ell(\theta;x,y)\|\le G_{\mathrm{grad}}$; (iii) for each round $q$, $\eta_q=\mathbb{E}_{x\sim Q_X^{(q)}}[\mathrm{TV}(\mathcal{P}(\cdot\mid x),\widetilde{\mathcal{P}}^{(q)}(\cdot\mid x))]$ and $\eta_{\mathrm{pl}}=\max_q\eta_q$; and (iv) whenever a client forms a public-gradient estimate at step $t$, it uses $B_{\mathrm{pub}}$ i.i.d. samples $\{(x_k,\tilde{Y}_k)\}_{k=1}^{B_{\mathrm{pub}}}$ with $x_k\sim Q_X^{(q(t))}$ and $\tilde{Y}_k\sim\widetilde{\mathcal{P}}^{(q(t))}(\cdot\mid x_k)$, independent of $\theta^t$ conditional on the past.
\end{assumption}

Define the high-probability public-gradient gap
\begin{equation}
\Delta_T(\rho)
=
G_{\mathrm{grad}}\sqrt{2\tau}
+
2G_{\mathrm{grad}}\eta_{\mathrm{pl}}
+
4G_{\mathrm{grad}}
\sqrt{
\frac{2}{B_{\mathrm{pub}}}
\left(
d_\theta\log 5+\log\frac{2T}{\rho}
\right)
}.
\label{eq:DeltaT_noniid}
\end{equation}

\begin{lemma}[Surrogate--reference public gradient gap]
\label{lem:sur_ref_gap_noniid}
Under Assumption~\ref{ass:shift_noise_noniid}, with probability at least $1-\rho$ over the sampling of public prompts and pseudo-labels across the $T$ SGD steps, $\|\nabla_\theta G_t(\theta^t)-\nabla_\theta G(\theta^t)\|\le\Delta_T(\rho)$ for all $t\in\{0,\dots,T-1\}$, where $G_t(\theta)$ is the empirical public objective formed from the $B_{\mathrm{pub}}$ samples in Assumption~\ref{ass:shift_noise_noniid}(4).
\end{lemma}

\begin{theorem}[Stationarity under shift, pseudo-label noise, and non-IID heterogeneity]
\label{thm:stationarity_noniid}
Let $\{\theta^t\}_{t=0}^{T}$ be produced by SGD, $\theta^{t+1}=\theta^t-\eta g^t$ for $t=0,\dots,T-1$. Under Assumptions~\ref{ass:smooth_sgd_noniid}, \ref{ass:heterogeneity_noniid}, and \ref{ass:shift_noise_noniid}, choose $\eta\le1/(4\Lambda)$ and fix $\rho\in(0,1)$. Then, on the event of Lemma~\ref{lem:sur_ref_gap_noniid},
\begin{equation}
\frac{1}{T}
\sum_{t=0}^{T-1}
\mathbb{E}\!\left[
\|\nabla \bar{\Phi}(\theta^t)\|^2
\right]
\le
\frac{
4\big(\bar{\Phi}(\theta^0)-\bar{\Phi}^\star\big)
}{
\eta T
}
+
2\Lambda\eta
\big(
\sigma^2+\frac{1}{4}\zeta^2
\big)
+
\frac{3}{4}
\Delta_T(\rho)^2,
\label{eq:main_stationarity_noniid}
\end{equation}
where $\bar{\Phi}^\star=\inf_\theta\bar{\Phi}(\theta)$ and the expectation is over SGD minibatch randomness.
\end{theorem}

\paragraph{Remark.}
The full proof is available in the appendix.\ref{app:convergence_full}. The term $\zeta^2/4$ captures client heterogeneity under the equal-weight global SGD abstraction. An analysis with multiple explicit local steps between communication rounds typically introduces additional drift terms depending on the number of local steps and the heterogeneity level.

\subsection{Privacy Analysis}
\label{sec:privacy}

Federated LoRA-based fine-tuning methods such as FlexLoRA and FLoRA keep raw client data local while aggregating adapter-side information at the server. Prior work shows that PEFT can be combined with differential privacy through mechanisms such as DP-SGD~\citep{liu2024differentially}. Motivated by this, we introduce DP-\fedcofit and show that \fedcofit admits a record-level differential privacy guarantee when each client's local LoRA fine-tuning mechanism is differentially private.

Let $\phi_i$ denote the trainable LoRA parameters of client $i$, with the backbone frozen. At communication round $q$, client $i$ applies a randomized local training mechanism $\phi_i^{(q)}\leftarrow\mathcal{A}_{i,q}(D_i,H^{(q-1)})$, where $H^{(q-1)}$ is the transcript up to round $q-1$. Assume that, for every fixed history $H^{(q-1)}$, $\mathcal{A}_{i,q}$ is $(\varepsilon_{i,q},\delta_{i,q})$-differentially private with respect to $D_i$. After local training, client $i$ generates responses on $D_{\mathrm{pub}}=\{x_j\}_{j=1}^{M}$ and releases $\mathcal{R}_i^{(q)}=\{\hat{y}_{i,j}^{(q)}\}_{j=1}^{M}$, where $\hat{y}_{i,j}^{(q)}\leftarrow\mathcal{M}_i(x_j;\phi_i^{(q)})$.

\begin{definition}[Neighboring datasets]
For a fixed client $i$, two datasets $D_i$ and $D_i'$ are neighboring if they differ in one training example.
\end{definition}

\begin{theorem}
\label{thm:dp-fedcofit}
Suppose that, for each communication round $q$, the client-side mechanism $\mathcal{A}_{i,q}(\cdot,H^{(q-1)})$ is $(\varepsilon_{i,q},\delta_{i,q})$-differentially private with respect to $D_i$. Then the released public-prompt responses $\mathcal{R}_i^{(q)}$ are also $(\varepsilon_{i,q},\delta_{i,q})$-differentially private with respect to $D_i$. Moreover, the overall multi-round mechanism induced by client $i$ over $Q$ communication rounds is $(\sum_{q=1}^{Q}\varepsilon_{i,q},\sum_{q=1}^{Q}\delta_{i,q})$-differentially private with respect to $D_i$.
\end{theorem}

\begin{proof}[Proof sketch]
At communication round $q$, the local LoRA fine-tuning mechanism is $(\varepsilon_{i,q},\delta_{i,q})$-DP by assumption. The released responses $\mathcal{R}_i^{(q)}$ depend only on the resulting private model state and the public prompts, so they are DP by post-processing. Server-side embedding, clustering, representative selection, and broadcast are also post-processing. Adaptive composition across communication rounds gives the stated bound.
\end{proof}

Details are given in appendix \ref{app:privacy} .In DP-\fedcofit, sharing responses on the public prompt set does not break the privacy guarantee, as those responses are outputs of a differentially private client-side fine-tuning mechanism. The guarantee is \emph{record-level} with respect to each client's local dataset.

\section{Empirical Evaluation}
\label{sec:experiments}

We evaluate \fedcofit along three axes: performance, communication, and efficiency. In all experiments, clients fine-tune locally on private data and collaborate through semantic consensus rounds using a shared public prompt set $D_{\mathrm{pub}}$ that was sampled from the dataset and excluded from the training data (the effect of this shared public prompt set size and distribution is studied in Appendix \ref{main:eff:promp:size} and \ref{app:prompt_distribution_shift}). We report task performance on held-out evaluation benchmarks, total communication volume measured as gigabytes uploaded/downloaded between clients and server across all rounds, and end-to-end wall-clock runtime and energy consumption. All methods are run with the same number of clients and rounds, and we report the mean score across clients. Full training hyperparameters, hardware details, and additional results are provided in Appendix~\ref{app:exp:details}. In the empirical evaluation, we consider three adaptation regimes: instruction compliance, knowledge generalization, and chat quality. For instruction compliance, clients fine-tune on Dolly-15k \cite{zhang2024towards}, and Alpaca \cite{taori2023stanford}. We evaluate instruction-following behaviour using IFEval benchmark \cite {zhou2023instruction}. For knowledge generalization, clients fine-tune on Wizard \cite{luo2023wizardmath} and OpenOrca \cite{OpenOrca}, we evaluate on MMLU benchmark \cite{hendrycks2020measuring}. For chat quality, clients fine-tune on ShareGPT \cite{chiang2023vicuna} and UltraChat \cite{ding2023enhancing} and evaluate on MT-Bench \cite{zheng2023judging}. For each setting, we compare \fedcofit against parameter aggregation baselines (e.g., FLoRA, FedIT, FlexLoRA) under matched federation budgets. We report final scores over rounds, together with total communication bytes (upload and download). To connect communication to practical cost, we additionally report wall-clock runtime and energy usage measured with Lamarr Energy Tracker \cite{ai_energy_validation}. Finally, we perform ablations to study the effect of design choices in public prompt size $|D_{\mathrm{pub}}|$, public prompt distribution, and the semantic encoder
\footnote{Code available at~\url{https://anonymous.4open.science/r/FedCoFiT}}.

\paragraph{Experiment results:} We first evaluate the performance of \fedcofit under a homogeneous LoRA setting, where all clients use the same LoRA rank $r=32$. For TinyLlama, LLaMA, and LLaMA2, we use $10$, $10$, and $5$ clients, respectively, with a micro-batch size of $16$. We additionally use a public prompt budget $M=500$ prompts sampled from the fine-tuning dataset.  Table \ref{tab:main_results_all} compares \fedcofit with parameter-aggregation baselines across the three regimes. We also report the total communication in GB. Two trends are consistent across all three foundation models.
First, \fedcofit matches or exceeds the strongest federated baseline in task quality. On TinyLlama, \fedcofit attains the best IFEval score on Dolly ($51.19$) and competitive performance across MMLU and MT-Bench benchmarks. On LLaMA-7B, it achieves the top score on IFEval ($54.36$ Dolly), MMLU($35.78$ OpenOrca), and on LLaMA2-13B, it achieved the highest score across most settings (e.g., MMLU Wizard $45.02$). These results indicate that behavior-level consensus provides learning signals that are effective.
Second, these quality levels are achieved with orders of magnitude lower communication. Across all experiments and baselines, \fedcofit communication budget is in the range of $\approx 0.001-0.007 GB$, where parameter aggregation-based methods require substantially larger transfers, e.g., FLoRA ranges from $\approx 4-21 GB$ and FlexLoRA from $\approx 0.36-3.75 GB$.  Concretely, on LLaMA-7B instruction tuning (Dolly), \fedcofit uses $0.006GB$ versus $3.13GB$ (FedIT), $3.75GB$ (FlexLoRA), and $20.63GB$ (FLoRA), while achieving the highest IFEval score. Similar gaps hold for MMLU and MT-Bench, demonstrating that \fedcofit delivers comparable accuracy at a much smaller federation budget.
Overall, \fedcofit is consistently competitive in performance while minimizing communication, and this communication reduction translates into practical system benefits. For example, \fedcofit reduces end-to-end runtime, energy consumption, and estimated $\mathrm{CO_2}$ emissions relative to parameter aggregation (Figure \ref{eff-tinyllama}, details are in Appendix \ref{app:eff:eval}). As model size grows and adapter updates become more expensive to transmit, these efficiency gains become increasingly pronounced.

\begin{table*}[t]
\centering
\small
\vspace{-0.8cm}
\setlength{\tabcolsep}{4pt}        
\renewcommand{\arraystretch}{1}

\caption{Comparison of Semantic Consensus \fedcofit with baselines across instruction compliance (IFEval), knowledge generalization (MMLU), and chat quality (MT-Bench). $GB$ denotes the cumulative communication volume (in gigabytes) summed across all clients and all rounds.}
\label{tab:main_results_all}

\resizebox{\textwidth}{!}{%
\begin{tabular}{ll cc cc cc cc cc cc cc cc cc}
\toprule
\multirow{3}{*}{\textbf{Foundation}} & \multirow{3}{*}{\textbf{Method}}
& \multicolumn{4}{c}{\textbf{IFEval} $\uparrow$}
& \multicolumn{4}{c}{\textbf{MMLU} $\uparrow$}
& \multicolumn{4}{c}{\textbf{MT-Bench} $\uparrow$} \\
\cmidrule(lr){3-6}\cmidrule(lr){7-10}\cmidrule(lr){11-14}
& 
& \multicolumn{2}{c}{\textbf{Dolly-15K}} & \multicolumn{2}{c}{\textbf{Alpaca}}
& \multicolumn{2}{c}{\textbf{Wizard}}   & \multicolumn{2}{c}{\textbf{OpenOrca}}
& \multicolumn{2}{c}{\textbf{ShareGPT}} & \multicolumn{2}{c}{\textbf{UltraChat}} \\
\cmidrule(lr){3-4}\cmidrule(lr){5-6}\cmidrule(lr){7-8}\cmidrule(lr){9-10}\cmidrule(lr){11-12}\cmidrule(lr){13-14}
& 
& \textbf{Score} & \textbf{GB} & \textbf{Score} & \textbf{GB}
& \textbf{Score} & \textbf{GB} & \textbf{Score} & \textbf{GB}
& \textbf{Score} & \textbf{GB} & \textbf{Score} & \textbf{GB} \\
\midrule

\multirow{5}{*}{\textbf{TinyLlama}}
& Centralized (LoRA)                  & $49.32$ & -- & $52.03$ & -- & $21.32$ & -- & $21.36$ & -- & $2.74$ & -- & $2.61$ & -- \\
& FedIT                               & $43.82$ & $0.9$ & $49.21$ &$0.9$ & $22.19$ & $0.54$ & $21.38$ & $0.54$ & $2.53$ & $0.18$ & $2.65$ & $0.18$ \\
& FLoRA                               & $50.26$ & $11.35$ & $51.10$ & $11.35$ & $22.25$ & $3.96$ & $22.84$ & $3.96$ & $2.81$ & $1.98$ & $3.10$ & $1.98$ \\
& FlexLoRA                            & $50.73$ & $1.2$ & $50.45$ & $1.2$ & $21.08$ & $0.72$ & $21.80$ & $0.72$ & $2.73$ & $0.36$ & $2.84$ & $0.36$ \\
& \textbf{\fedcofit (ours)}  & $51.19$ & \textbf{$0.007$} & \textbf{$50.71$} & \textbf{$0.002$}
                                      & \textbf{$22.60$} & \textbf{$0.005$} & \textbf{$22.73$} & \textbf{$0.001$}
                                      & \textbf{$2.79$} & \textbf{$0.002$} & \textbf{$2.97$} & \textbf{$0.002$} \\
\midrule

\multirow{5}{*}{\textbf{LLaMA}}
& Centralized (LoRA)                  & $54.18$ & -- & $55.17$ & -- & $34.28$ & -- & $35.19$ & -- & $3.34$ & -- & $3.25$ & -- \\
& FedIT                               & $50.37$ & $3.13$ & $50.92$& $3.13$& $31.46$ & $1.88$ & $31.74$ & $1.88$ & $3.29$ & $0.625$ & $3.17$ & $0.625$ \\
& FLoRA                               & $53.29$ & $20.63$ & $54.91$ & $20.63$ & $34.92$ & $13.75$ & $35.21$ & $13.75$ & $3.41$ & $6.875$ & $3.31$ & $6.875$ \\
& FlexLoRA                            & $52.61$ & $3.75$ & $53.79$ & $3.75$ & $34.85$ & $2.5$ & $34.62$ & $2.5$ & $3.36$ & $1.25$ & $3.27$ & $1.25$ \\
& \textbf{\fedcofit (ours)}  & \textbf{$54.36$} & \textbf{$0.006$} & \textbf{$54.83$} & \textbf{$0.004$}
                                      & \textbf{$34.89$} & \textbf{$0.004$} & \textbf{$35.78$} & \textbf{$0.001$}
                                      & \textbf{$3.45$} & \textbf{$0.002$} & \textbf{$3.34$} & \textbf{$0.002$} \\
\midrule

\multirow{5}{*}{\textbf{LLaMA2}}
& Centralized (LoRA)                  & $58.71$ & -- & $59.24$ & -- & $43.27$ & -- & $43.85$ & -- &$5.86$ & -- & $6.13$& -- \\
& FedIT                               & $52.38$ & $2.62$ & $53.29$ & $2.62$ & $43.29$ & $1.57$ & $42.06$ & $1.57$ & $5.93$ & $0.52$ & $6.03$ & $0.52$ \\
& FLoRA                               & $58.49$ & $17.30$ & $60.25$ & $17.30$ & $44.57$ & $11.53$ & $44.52$ & $11.53$ & $6.03$ & $5.77$ & $6.21$ & $5.77$ \\
& FlexLoRA                            & $57.24$ & $3.14$ & $59.50$ & $3.14$ & $44.29$ & $2.1$ & $43.96$ & $2.1$ & $5.97$ & $1.05$ & $6.14$ & $1.05$ \\
& \textbf{\fedcofit (ours)}  & \textbf{$57.96$} & \textbf{$0.006$} & \textbf{$60.31$} & \textbf{$0.004$}
                                      & \textbf{$45.02$} & \textbf{$0.004$} & \textbf{$44.28$} & \textbf{$0.001$}
                                      & \textbf{$6.12$} & \textbf{$0.002$} & \textbf{$6.19$} & \textbf{$0.003$} \\

\bottomrule
\end{tabular}%
}
\end{table*}

\paragraph{Heterogeneous LoRA Ranks:}
%
In many federated deployments, clients cannot be assumed to use identical PEFT configurations due to different compute, memory, or personalization needs. To evaluate robustness under such heterogeneity, we consider a heterogeneous LoRA-rank setting in which clients fine-tune the same TinyLlama backbone but use different LoRA ranks. In this setting, we report task scores. Let $r_i$ denote the LoRA rank used by client $i$. We assign ranks across clients by sampling from a small set of representative ranks $r_i \in \{8,16,32\}$ and distributing across the $10$ clients. All other hyperparameters (target modules, learning rate, batch size, rounds, and public prompt budget) are kept identical to the homogeneous setting. Table~\ref{tab:heter_rank_tinyllama} reports performance under heterogeneous ranks across instruction following (IFEval), knowledge generalization (MMLU), and chat quality (MT-Bench). FedIT is marked as N/A because it does not support heterogeneous LoRA ranks. Parameter-aggregation such as FLoRA and FlexLoRA baselines degrade under rank heterogeneity because adapter shapes differ across clients. In contrast, \fedcofit aggregates client behaviour in prediction space and is therefore agnostic to rank choices. Empirically, \fedcofit remains stable under heterogeneous ranks and achieves the best or competitive scores across tasks, demonstrating that behaviour-level consensus naturally supports heterogeneous PEFT configurations.

\begin{table}[t]
\vspace{-0.5cm}
\centering
\small
\setlength{\tabcolsep}{10pt}
\renewcommand{\arraystretch}{1.05}

\caption{TinyLlama results under \textbf{heterogeneous LoRA ranks}. Clients use ranks $r_i \in \{8,16,32\}$ distributed across $K{=}10$ clients.}
\label{tab:heter_rank_tinyllama}

\resizebox{\columnwidth}{!}{%
\begin{tabular}{l cc cc cc}
\toprule
\multirow{2}{*}{\textbf{Method}}
& \multicolumn{2}{c}{\textbf{IFEval} $\uparrow$}
& \multicolumn{2}{c}{\textbf{MMLU} $\uparrow$}
& \multicolumn{2}{c}{\textbf{MT-Bench} $\uparrow$} \\
\cmidrule(lr){2-3}\cmidrule(lr){4-5}\cmidrule(lr){6-7}
& \textbf{Dolly-15K} & \textbf{Alpaca}
& \textbf{Wizard} & \textbf{OpenOrca}
& \textbf{ShareGPT} & \textbf{UltraChat} \\
\midrule
FedIT                      & \textit{N/A} & \textit{N/A} &  & \textit{N/A} & \textit{N/A} & \textit{N/A} \\
FLoRA                      & $49.87$ & $50.26$ & $20.93$ & $21.04$ &  $2.58$& $2.61$ \\
FlexLoRA                   & $49.38$ & $50.12$ & $20.19$ & $20.93$ & $2.46$& $2.57$ \\
\textbf{\fedcofit (ours)}  & $51.20$ & $50.39$ & $21.34$ & $22.08$ & $2.74$ & $2.84$ \\
\bottomrule
\end{tabular}%
}
\end{table}
\paragraph{Ablations and Robustness:}
We further evaluate the robustness of \fedcofit to the design of the public prompt set. 
First, we vary the public prompt budget $M=|D_{\mathrm{pub}}|$ and find that performance is stable for moderate prompt budgets, while very small public sets reduce consensus quality (Appendix~\ref{main:eff:promp:size}). Second, we test public-prompt distribution shift by replacing in-distribution Dolly prompts with prompts from Alpaca, OpenOrca, and ShareGPT; \fedcofit remains robust to instruction-style shifts but degrades under stronger chat-format shift (Appendix~\ref{app:prompt_distribution_shift}). We also ablate the consensus selection rule in Appendix \ref{app:representative_selection_ablation}. Finally, Appendix~\ref{app:eff:eval} reports wall-clock, energy, and $\mathrm{CO}_2$ efficiency results, and Appendix~\ref{app:encoder_ablation} ablates the semantic encoder choice, motivating our default use of \texttt{all-MiniLM-L6-v2}.

\begin{figure*}
    \centering
    \includegraphics[width=0.9\linewidth]{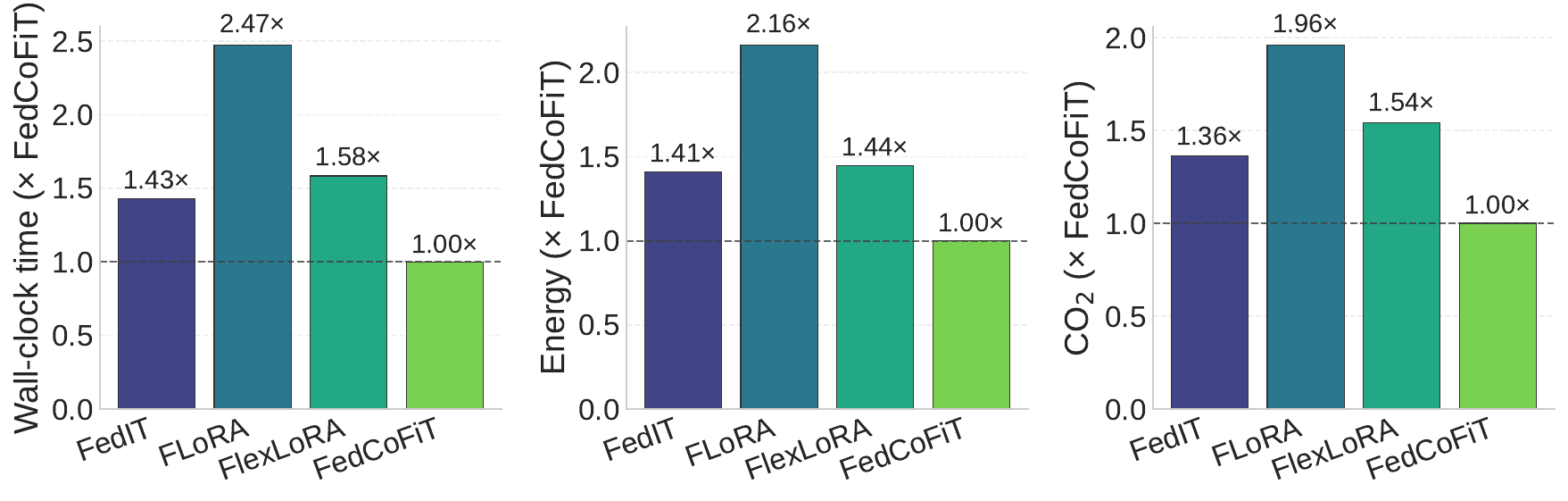}
    \caption{Relative efficiency comparison on Dolly-15K using  LLaMA-7B}
    \label{eff-tinyllama}
\end{figure*}

\section{Discussion and Conclusion}
\label{sec:discussion}
This work introduces Semantic Consensus as a fundamental principle for federated fine-tuning of large language models, in which collaboration is mediated through shared model behaviours rather than shared parameters. \fedcofit shows that clients can collaborate by exchanging responses to public prompts, allowing the server to form consensus pseudo-labels without sharing weights, gradients, or adapters. This directly reduces communication, supports heterogeneous architectures, and remains applicable when clients have different LoRA configurations.

The theoretical analysis shows that \fedcofit communication cost depends on the number of public prompts and response length, not model size, giving a fundamentally different scaling law from LoRA or parameter aggregation approaches. Furthermore, we theoretically show that \fedcofit converges with high probability under standard assumptions and identify the main factors that affect its convergence quality. Our privacy analysis shows that if local client fine-tuning is differentially private, then the released public-prompt responses and server-side consensus steps preserve record-level differential privacy through post-processing and composition.

Empirically, FedCoFiT matches or exceeds strong federated LoRA baselines across instruction following, knowledge generalization, and chat-quality tasks while using orders of magnitude less communication. It also remains effective under heterogeneous LoRA ranks, where parameter aggregation is difficult or unsupported. As shown in our Ablations, the main limitation of \fedcofit is dependence on the public prompt set and semantic clustering quality, as performance can degrade when prompts are too few or poorly aligned with the target task.
Overall, Semantic Consensus suggests that behavior-level sharing is a promising alternative to parameter aggregation for federated adaptation of large and heterogeneous LLMs. Future work should improve prompt selection, robustness under distribution shift, and extensions to behaviors beyond text generation.

\vfill
\pagebreak
\bibliography{bibliography}

@inproceedings{zhang2024towards,
  title={Towards building the federatedgpt: Federated instruction tuning},
  author={Zhang, Jianyi and Vahidian, Saeed and Kuo, Martin and Li, Chunyuan and Zhang, Ruiyi and Yu, Tong and Wang, Guoyin and Chen, Yiran},
  booktitle={ICASSP 2024-2024 IEEE International Conference on Acoustics, Speech and Signal Processing (ICASSP)},
  pages={6915--6919},
  year={2024},
  organization={IEEE}
}

@inproceedings{ma2023fedid,
  title={Fedid: Federated interactive distillation for large-scale pretraining language models},
  author={Ma, Xinge and Liu, Jiangming and Wang, Jin and Zhang, Xuejie},
  booktitle={Proceedings of the 2023 Conference on Empirical Methods in Natural Language Processing},
  pages={8566--8577},
  year={2023}
}

@article{hu2022lora,
  title={Lora: Low-rank adaptation of large language models.},
  author={Hu, Edward J and Shen, Yelong and Wallis, Phillip and Allen-Zhu, Zeyuan and Li, Yuanzhi and Wang, Shean and Wang, Lu and Chen, Weizhu and others},
  journal={ICLR},
  volume={1},
  number={2},
  pages={3},
  year={2022}
}

@inproceedings{mcmahan2017communication,
  title={Communication-efficient learning of deep networks from decentralized data},
  author={McMahan, Brendan and Moore, Eider and Ramage, Daniel and Hampson, Seth and y Arcas, Blaise Aguera},
  booktitle={Artificial intelligence and statistics},
  pages={1273--1282},
  year={2017},
  organization={PMLR}
}

@article{wang2024flora,
  title={Flora: Federated fine-tuning large language models with heterogeneous low-rank adaptations},
  author={Wang, Ziyao and Shen, Zheyu and He, Yexiao and Sun, Guoheng and Wang, Hongyi and Lyu, Lingjuan and Li, Ang},
  journal={Advances in Neural Information Processing Systems},
  volume={37},
  pages={22513--22533},
  year={2024}
}

@article{bai2024federated,
  title={Federated fine-tuning of large language models under heterogeneous tasks and client resources},
  author={Bai, Jiamu and Chen, Daoyuan and Qian, Bingchen and Yao, Liuyi and Li, Yaliang},
  journal={Advances in Neural Information Processing Systems},
  volume={37},
  pages={14457--14483},
  year={2024}
}

@article{cho2024heterogeneous,
  title={Heterogeneous lora for federated fine-tuning of on-device foundation models},
  author={Cho, Yae Jee and Liu, Luyang and Xu, Zheng and Fahrezi, Aldi and Joshi, Gauri},
  journal={arXiv preprint arXiv:2401.06432},
  year={2024}
}

@inproceedings{bian2025lora,
  title={Lora-fair: Federated lora fine-tuning with aggregation and initialization refinement},
  author={Bian, Jieming and Wang, Lei and Zhang, Letian and Xu, Jie},
  booktitle={Proceedings of the IEEE/CVF International Conference on Computer Vision},
  pages={3737--3746},
  year={2025}
}

@inproceedings{koo2025towards,
  title={Towards robust and efficient federated low-rank adaptation with heterogeneous clients},
  author={Koo, Jabin and Jang, Minwoo and Ok, Jungseul},
  booktitle={Proceedings of the 63rd Annual Meeting of the Association for Computational Linguistics (Volume 1: Long Papers)},
  pages={416--429},
  year={2025}
}

@article{sun2023fedbpt,
  title={Fedbpt: Efficient federated black-box prompt tuning for large language models},
  author={Sun, Jingwei and Xu, Ziyue and Yin, Hongxu and Yang, Dong and Xu, Daguang and Chen, Yiran and Roth, Holger R},
  journal={arXiv preprint arXiv:2310.01467},
  year={2023}
}

@misc{taori2023stanford,
  title={Stanford alpaca: An instruction-following llama model},
  author={Taori, Rohan and Gulrajani, Ishaan and Zhang, Tianyi and Dubois, Yann and Li, Xuechen and Guestrin, Carlos and Liang, Percy and Hashimoto, Tatsunori B},
  year={2023},
  publisher={Stanford, CA, USA}
}

@article{zhou2023instruction,
  title={Instruction-following evaluation for large language models},
  author={Zhou, Jeffrey and Lu, Tianjian and Mishra, Swaroop and Brahma, Siddhartha and Basu, Sujoy and Luan, Yi and Zhou, Denny and Hou, Le},
  journal={arXiv preprint arXiv:2311.07911},
  year={2023}
}

@article{luo2023wizardmath,
  title={Wizardmath: Empowering mathematical reasoning for large language models via reinforced evol-instruct},
  author={Luo, Haipeng and Sun, Qingfeng and Xu, Can and Zhao, Pu and Lou, Jianguang and Tao, Chongyang and Geng, Xiubo and Lin, Qingwei and Chen, Shifeng and Zhang, Dongmei},
  journal={arXiv preprint arXiv:2308.09583},
  year={2023}
}

@misc{OpenOrca,
  title = {OpenOrca: An Open Dataset of GPT Augmented FLAN Reasoning Traces},
  author = {Wing Lian and Bleys Goodson and Eugene Pentland and Austin Cook and Chanvichet Vong and "Teknium"},
  year = {2023},
  publisher = {HuggingFace},
  journal = {HuggingFace repository},
  howpublished = {\url{https://https://huggingface.co/datasets/Open-Orca/OpenOrca}},
}

@article{hendrycks2020measuring,
  title={Measuring massive multitask language understanding},
  author={Hendrycks, Dan and Burns, Collin and Basart, Steven and Zou, Andy and Mazeika, Mantas and Song, Dawn and Steinhardt, Jacob},
  journal={arXiv preprint arXiv:2009.03300},
  year={2020}
}

@article{chiang2023vicuna,
  title={Vicuna: An open-source chatbot impressing gpt-4 with 90\%* chatgpt quality},
  author={Chiang, Wei-Lin and Li, Zhuohan and Lin, Ziqing and Sheng, Ying and Wu, Zhanghao and Zhang, Hao and Zheng, Lianmin and Zhuang, Siyuan and Zhuang, Yonghao and Gonzalez, Joseph E and others},
  journal={See https://vicuna. lmsys. org (accessed 14 April 2023)},
  volume={2},
  number={3},
  pages={6},
  year={2023}
}

@inproceedings{kamp2018efficient,
  title={Efficient decentralized deep learning by dynamic model averaging},
  author={Kamp, Michael and Adilova, Linara and Sicking, Joachim and H{\"u}ger, Fabian and Schlicht, Peter and Wirtz, Tim and Wrobel, Stefan},
  booktitle={Joint European conference on machine learning and knowledge discovery in databases},
  pages={393--409},
  year={2018},
  organization={Springer}
}

@inproceedings{kamp2016communication,
  title={Communication-efficient distributed online learning with kernels},
  author={Kamp, Michael and Bothe, Sebastian and Boley, Mario and Mock, Michael},
  booktitle={ECMLPKDD},
  pages={805--819},
  year={2016},
  organization={Springer}
}

@phdthesis{kamp2019black,
  author       = {Michael Kamp}, 
  title        = {Black-Box Parallelization for Machine Learning},
  school       = {Rheinische Friedrich-Wilhelms-Universit{\"a}t Bonn},
  year         = 2019,
  address      = {Universit{\"a}ts-und Landesbibliothek Bonn},
}

@article{ding2023enhancing,
  title={Enhancing Chat Language Models by Scaling High-quality Instructional Conversations},
  author={Ding, Ning and Chen, Yulin and Xu, Bokai and Qin, Yujia and Zheng, Zhi and Hu, Shengding and Liu, Zhiyuan and Sun, Maosong and Zhou, Bowen},
  journal={arXiv preprint arXiv:2305.14233},
  year={2023}
}

@article{zheng2023judging,
  title={Judging llm-as-a-judge with mt-bench and chatbot arena},
  author={Zheng, Lianmin and Chiang, Wei-Lin and Sheng, Ying and Zhuang, Siyuan and Wu, Zhanghao and Zhuang, Yonghao and Lin, Zi and Li, Zhuohan and Li, Dacheng and Xing, Eric and others},
  journal={Advances in neural information processing systems},
  volume={36},
  pages={46595--46623},
  year={2023}
}

@misc{ai_energy_validation,
  title  = {Ground-Truthing {AI} Energy Consumption: {Validating} {CodeCarbon} Against External Measurements}, 
  author = {Raphael Fischer},
  year   = {2025},
  doi    = {10.48550/arXiv.2509.22092},
  url    = {https://arxiv.org/abs/2509.22092}, 
}

@article{ghadimi2013stochastic,
  title   = {Stochastic First- and Zeroth-Order Methods for Nonconvex Stochastic Programming},
  author  = {Ghadimi, Saeed and Lan, Guanghui},
  journal = {SIAM Journal on Optimization},
  volume  = {23},
  number  = {4},
  pages   = {2341--2368},
  year    = {2013},
  doi     = {10.1137/120880811},
  eprint  = {1309.5549},
  archivePrefix = {arXiv},
  primaryClass  = {math.OC}
}

@inproceedings{karimireddy2020scaffold,
  title     = {SCAFFOLD: Stochastic Controlled Averaging for Federated Learning},
  author    = {Karimireddy, Sai Praneeth and Kale, Satyen and Mohri, Mehryar and
               Reddi, Sashank J. and Stich, Sebastian U. and Suresh, Ananda Theertha},
  booktitle = {Proceedings of the 37th International Conference on Machine Learning},
  series    = {Proceedings of Machine Learning Research},
  volume    = {119},
  pages     = {5132--5143},
  year      = {2020},
  publisher = {PMLR},
  url       = {https://proceedings.mlr.press/v119/karimireddy20a.html}
}

@inproceedings{li2020fedavg,
  title     = {On the Convergence of FedAvg on Non-IID Data},
  author    = {Li, Xiang and Huang, Kaixuan and Yang, Wenhao and Wang, Shusen and Zhang, Zhihua},
  booktitle = {International Conference on Learning Representations},
  year      = {2020},
  url       = {https://openreview.net/forum?id=HJxNAnVtDS},
  eprint    = {1907.02189},
  archivePrefix = {arXiv},
  primaryClass  = {cs.LG}
}

@inproceedings{abadi2016deep,
  title     = {Deep Learning with Differential Privacy},
  author    = {Abadi, Mart{\'i}n and Chu, Andy and Goodfellow, Ian and McMahan, H. Brendan and
               Mironov, Ilya and Talwar, Kunal and Zhang, Li},
  booktitle = {Proceedings of the 2016 ACM SIGSAC Conference on Computer and Communications Security},
  pages     = {308--318},
  year      = {2016},
  doi       = {10.1145/2976749.2978318},
  eprint    = {1607.00133},
  archivePrefix = {arXiv},
  primaryClass  = {cs.LG}
}

@inproceedings{chen2020understanding,
  title     = {Understanding Gradient Clipping in Private SGD: A Geometric Perspective},
  author    = {Chen, Xiangyi and Wu, Zhiwei Steven and Hong, Mingyi},
  booktitle = {Advances in Neural Information Processing Systems},
  year      = {2020},
  url       = {https://proceedings.neurips.cc/paper/2020/hash/9ecff5455677b38d19f49ce658ef0608-Abstract.html},
  eprint    = {2006.15429},
  archivePrefix = {arXiv},
  primaryClass  = {cs.LG}
}

@article{li2018fedprox,
  title   = {Federated Optimization in Heterogeneous Networks},
  author  = {Li, Tian and Sahu, Anit Kumar and Zaheer, Manzil and Sanjabi, Maziar and Talwalkar, Ameet and Smith, Virginia},
  year    = {2018},
  eprint  = {1812.06127},
  archivePrefix = {arXiv},
  primaryClass  = {cs.LG},
  note    = {Also appeared in Proceedings of Machine Learning and Systems (MLSys) 2020.}
}

@article{sugiyama2007covariate,
  title   = {Covariate Shift Adaptation by Importance Weighted Cross Validation},
  author  = {Sugiyama, Masashi and Krauledat, Matthias and M{\"u}ller, Klaus-Robert},
  journal = {Journal of Machine Learning Research},
  volume  = {8},
  number  = {35},
  pages   = {985--1005},
  year    = {2007},
  url     = {https://www.jmlr.org/papers/v8/sugiyama07a.html}
}

@inproceedings{aminian2022information,
  title     = {An Information-theoretical Approach to Semi-supervised Learning under Covariate-shift},
  author    = {Aminian, Gholamali and Abroshan, Mahed and Khalili, Mohammad Mahdi and Toni, Laura and Rodrigues, Miguel},
  booktitle = {Proceedings of The 25th International Conference on Artificial Intelligence and Statistics},
  series    = {Proceedings of Machine Learning Research},
  volume    = {151},
  pages     = {7433--7449},
  year      = {2022},
  publisher = {PMLR},
  url       = {https://proceedings.mlr.press/v151/aminian22a.html},
  eprint    = {2202.12123},
  archivePrefix = {arXiv},
  primaryClass  = {cs.LG}
}

@inproceedings{natarajan2013learning,
  title     = {Learning with Noisy Labels},
  author    = {Natarajan, Nagarajan and Dhillon, Inderjit S. and Ravikumar, Pradeep K. and Tewari, Ambuj},
  booktitle = {Advances in Neural Information Processing Systems},
  pages     = {1196--1204},
  year      = {2013},
  url       = {https://papers.nips.cc/paper/5073-learning-with-noisy-labels}
}

@inproceedings{wei2020theoretical,
  title     = {Theoretical Analysis of Self-Training with Deep Networks on Unlabeled Data},
  author    = {Wei, Colin and Shen, Kendrick and Chen, Yining and Ma, Tengyu},
  booktitle = {International Conference on Learning Representations},
  year      = {2021},
  url       = {https://openreview.net/forum?id=rC8sJ4i6kaH},
  eprint    = {2010.03622},
  archivePrefix = {arXiv},
  primaryClass  = {cs.LG}
}

@article{liu2024differentially,
  title={Differentially private parameter-efficient fine-tuning for large asr models},
  author={Liu, Hongbin and Wang, Lun and Thakkar, Om and Thakurta, Abhradeep and Narayanan, Arun},
  journal={arXiv preprint arXiv:2410.01948},
  year={2024}
}

@article{zhang2026heterogeneous,
  title={Heterogeneous Federated Fine-Tuning with Parallel One-Rank Adaptation},
  author={Zhang, Zikai and Hu, Rui and Xu, Jiahao},
  journal={arXiv preprint arXiv:2602.16936},
  year={2026}
}

@article{huang2026stabilized,
  title={Stabilized Fine-Tuning with LoRA in Federated Learning: Mitigating the Side Effect of Client Size and Rank via the Scaling Factor},
  author={Huang, Jiayu and Wu, Xiaohu and He, Tiantian and Lao, Qicheng},
  journal={arXiv preprint arXiv:2603.08058},
  year={2026}
}

@inproceedings{abourayya2025little,
  title={Little is enough: Boosting privacy by sharing only hard labels in federated semi-supervised learning},
  author={Abourayya, Amr and Kleesiek, Jens and Rao, Kanishka and Ayday, Erman and Rao, Bharat and Webb, Geoffrey I and Kamp, Michael},
  booktitle={Proceedings of the AAAI Conference on Artificial Intelligence},
  volume={39},
  number={15},
  pages={15293--15301},
  year={2025}
}

@article{bistritz2020distributed,
  title={Distributed distillation for on-device learning},
  author={Bistritz, Ilai and Mann, Ariana and Bambos, Nicholas},
  journal={Advances in Neural Information Processing Systems},
  volume={33},
  pages={22593--22604},
  year={2020}
}
\bibliographystyle{plainnat}
\vfill

\appendix
\pagebreak
\onecolumn
\section{Convergence Proofs and Supporting Lemmas}
\label{app:convergence_full}

\subsection{Proof of Lemma~\ref{lem:sur_ref_gap_noniid}}
\label{app:proof_sur_ref_gap_noniid}

Fix a parameter $\theta\in\mathbb{R}^{d_\theta}$ and a communication round $q$. Define the population risks using the global reference distribution $\mathcal{P}$:
\[
R_{\mathcal{P}}(\theta):=\mathbb{E}_{x\sim \mathcal{P}_X}\mathbb{E}_{y\sim \mathcal{P}(\cdot\mid x)}[\ell(\theta;x,y)] \quad (=G(\theta)),
\]
\[
R_Q^{(q)}(\theta):=\mathbb{E}_{x\sim Q_X^{(q)}}\mathbb{E}_{y\sim \mathcal{P}(\cdot\mid x)}[\ell(\theta;x,y)],
\qquad
\widetilde R_Q^{(q)}(\theta):=\mathbb{E}_{x\sim Q_X^{(q)}}\mathbb{E}_{\tilde{Y}\sim \widetilde{\mathcal{P}}^{(q)}(\cdot\mid x)}[\ell(\theta;x,\tilde{Y})].
\]
For any $\theta$ and $x$, define $v_\theta(x):=\mathbb{E}_{y\sim \mathcal{P}(\cdot\mid x)}[\nabla_\theta \ell(\theta;x,y)]$ and $\widetilde v_\theta^{(q)}(x):=\mathbb{E}_{\tilde{Y}\sim \widetilde{\mathcal{P}}^{(q)}(\cdot\mid x)}[\nabla_\theta \ell(\theta;x,\tilde{Y})]$. By Assumption~\ref{ass:shift_noise_noniid}(2), $\|v_\theta(x)\|\le G_{\mathrm{grad}}$ and $\|\widetilde v_\theta^{(q)}(x)\|\le G_{\mathrm{grad}}$.

\paragraph{A. Covariate shift bound.}

\begin{lemma}[Pinsker]
\label{lem:pinsker_app_noniid}
For distributions $\mu,\nu$ on the same space, $\mathrm{TV}(\mu,\nu)\le \sqrt{\frac12\,\mathrm{KL}(\mu\|\nu)}$.
\end{lemma}

\begin{lemma}[Expectation shift under total variation]
\label{lem:tv_expect_app_noniid}
Let $h:\mathcal{X}\to\mathbb{R}^{d_\theta}$ be measurable and satisfy $\|h(x)\|\le B$ for all $x$. Then $\|\mathbb{E}_{x\sim Q}[h(x)]-\mathbb{E}_{x\sim P}[h(x)]\|\le 2B\,\mathrm{TV}(P,Q)$.
\end{lemma}

\begin{lemma}[Covariate shift gradient bound]
\label{lem:cov_shift_grad_app_noniid}
If $\mathrm{KL}(\mathcal{P}_X\|Q_X^{(q)})\le\tau_q$ and $\|v_\theta(x)\|\le G_{\mathrm{grad}}$, then
\[
\|\nabla R_Q^{(q)}(\theta)-\nabla R_{\mathcal{P}}(\theta)\|
\le
G_{\mathrm{grad}}\sqrt{2\tau_q}
\le
G_{\mathrm{grad}}\sqrt{2\tau}.
\]
\end{lemma}

\begin{proof}
Since $\nabla R_Q^{(q)}(\theta)=\mathbb{E}_{Q_X^{(q)}}[v_\theta(x)]$ and $\nabla R_{\mathcal{P}}(\theta)=\mathbb{E}_{\mathcal{P}_X}[v_\theta(x)]$, Lemma~\ref{lem:tv_expect_app_noniid} with $h=v_\theta$ and $B=G_{\mathrm{grad}}$ gives $\|\nabla R_Q^{(q)}(\theta)-\nabla R_{\mathcal{P}}(\theta)\|\le2G_{\mathrm{grad}}\,\mathrm{TV}(\mathcal{P}_X,Q_X^{(q)})$. Pinsker and $\mathrm{KL}(\mathcal{P}_X\|Q_X^{(q)})\le\tau_q$ imply $\mathrm{TV}(\mathcal{P}_X,Q_X^{(q)})\le \sqrt{\tau_q/2}$, yielding the claim.
\end{proof}

\paragraph{B. Pseudo-label noise bound.}

\begin{lemma}[Pseudo-label noise gradient bound]
\label{lem:label_noise_grad_app_noniid}
Under Assumption~\ref{ass:shift_noise_noniid}, for all $\theta$ and each communication round $q$,
\[
\|\nabla \widetilde R_Q^{(q)}(\theta)-\nabla R_Q^{(q)}(\theta)\|
\le
2G_{\mathrm{grad}}\eta_q
\le
2G_{\mathrm{grad}}\eta_{\mathrm{pl}}.
\]
\end{lemma}

\begin{proof}
For fixed $\theta$ and $x$, $\widetilde v_\theta^{(q)}(x)-v_\theta(x)=\mathbb{E}_{\tilde{Y}\sim\widetilde{\mathcal{P}}^{(q)}(\cdot\mid x)}[\nabla_\theta\ell(\theta;x,\tilde{Y})]-\mathbb{E}_{y\sim\mathcal{P}(\cdot\mid x)}[\nabla_\theta\ell(\theta;x,y)]$. Since $\|\nabla_\theta\ell(\theta;x,y)\|\le G_{\mathrm{grad}}$, the total-variation expectation bound gives $\|\widetilde v_\theta^{(q)}(x)-v_\theta(x)\|\le2G_{\mathrm{grad}}\mathrm{TV}(\mathcal{P}(\cdot\mid x),\widetilde{\mathcal{P}}^{(q)}(\cdot\mid x))$. Taking $\mathbb{E}_{x\sim Q_X^{(q)}}$ yields the result.
\end{proof}

\paragraph{C. Finite-sample concentration for the empirical public gradient.}

At a public-gradient evaluation using communication round $q$, let $\{(x_k,\tilde{Y}_k)\}_{k=1}^{B_{\mathrm{pub}}}$ be i.i.d. with $x_k\sim Q_X^{(q)}$ and $\tilde{Y}_k\sim\widetilde{\mathcal{P}}^{(q)}(\cdot\mid x_k)$. Define $\widehat g^{(q)}(\theta):=B_{\mathrm{pub}}^{-1}\sum_{k=1}^{B_{\mathrm{pub}}}\nabla_\theta \ell(\theta;x_k,\tilde{Y}_k)$, so $\mathbb{E}[\widehat g^{(q)}(\theta)]=\nabla \widetilde R_Q^{(q)}(\theta)$.

\begin{lemma}[Vector Hoeffding via a $1/2$-net]
\label{lem:net_hoeff_app_noniid}
Fix $\theta$ and $\rho\in(0,1)$. If $\|\nabla_\theta \ell(\theta;x,y)\|\le G_{\mathrm{grad}}$, then with probability at least $1-\rho$,
\[
\|\widehat g^{(q)}(\theta)-\nabla \widetilde R_Q^{(q)}(\theta)\|
\le
4G_{\mathrm{grad}}
\sqrt{
\frac{2}{B_{\mathrm{pub}}}
\left(
d_\theta\log 5+\log\frac{2}{\rho}
\right)
}.
\]
\end{lemma}

\begin{proof}
Construct a $1/2$-net $\mathcal{N}$ of $S^{d_\theta-1}$ with $|\mathcal{N}|\le 5^{d_\theta}$. For each $u\in\mathcal{N}$, apply scalar Hoeffding to $\langle u,\widehat g^{(q)}(\theta)-\nabla\widetilde R_Q^{(q)}(\theta)\rangle$ and union bound over $\mathcal{N}$. Lift from net control to full norm control by the standard approximation argument.
\end{proof}

\paragraph{D. Combine and union bound over $T$ steps.}

Fix step $t$ and write $q=q(t)$. By the triangle inequality,
\[
\begin{aligned}
\|\widehat g^{(q)}(\theta^t)-\nabla R_{\mathcal{P}}(\theta^t)\|
\le\;&
\|\widehat g^{(q)}(\theta^t)-\nabla \widetilde R_Q^{(q)}(\theta^t)\|
+
\|\nabla \widetilde R_Q^{(q)}(\theta^t)-\nabla R_Q^{(q)}(\theta^t)\|
\\
&+
\|\nabla R_Q^{(q)}(\theta^t)-\nabla R_{\mathcal{P}}(\theta^t)\|.
\end{aligned}
\]
Lemma~\ref{lem:net_hoeff_app_noniid} with $\rho_t$, Lemma~\ref{lem:label_noise_grad_app_noniid}, and Lemma~\ref{lem:cov_shift_grad_app_noniid} imply that with probability at least $1-\rho_t$,
\[
\|\widehat g^{(q)}(\theta^t)-\nabla R_{\mathcal{P}}(\theta^t)\|
\le
G_{\mathrm{grad}}\sqrt{2\tau}
+
2G_{\mathrm{grad}}\eta_{\mathrm{pl}}
+
4G_{\mathrm{grad}}
\sqrt{
\frac{2}{B_{\mathrm{pub}}}
\left(
d_\theta\log 5+\log\frac{2}{\rho_t}
\right)
}.
\]
Set $\rho_t=\rho/T$ and union bound over $t=0,\dots,T-1$ to obtain, with probability at least $1-\rho$, $\|\widehat g^{(q(t))}(\theta^t)-\nabla R_{\mathcal{P}}(\theta^t)\|\le\Delta_T(\rho)$ for all $t$. Identifying $\nabla_\theta G_t(\theta^t)$ with $\widehat g^{(q(t))}(\theta^t)$ and $\nabla_\theta G(\theta^t)$ with $\nabla R_{\mathcal{P}}(\theta^t)$ proves Lemma~\ref{lem:sur_ref_gap_noniid}.
\qed

\subsection{A Supporting Lemma: Heterogeneity Contributes Extra Variance}
\label{app:heterogeneity_variance_lemma}

\begin{lemma}[Client sampling variance bound]
\label{lem:client_var}
Fix $\theta\in\mathbb{R}^{d_\theta}$ and suppose a client index $I$ is sampled uniformly from $\{1,\dots,K\}$. Let $\widehat{\nabla} F_I(\theta)$ be a possibly minibatched unbiased estimator of $\nabla F_I(\theta)$ satisfying $\mathbb{E}[\widehat{\nabla} F_i(\theta)]=\nabla F_i(\theta)$ and $\mathbb{E}[\|\widehat{\nabla} F_i(\theta)-\nabla F_i(\theta)\|^2]\le \sigma_{\mathrm{priv}}^2$ for all $i$. Then $\widehat{\nabla} F_I(\theta)$ is unbiased for $\nabla F(\theta)=K^{-1}\sum_{i=1}^{K}\nabla F_i(\theta)$ and
\[
\mathbb{E}\big[\|\widehat{\nabla} F_I(\theta)-\nabla F(\theta)\|^2\big]
\le
\sigma_{\mathrm{priv}}^2
+
\frac{1}{K}\sum_{i=1}^{K}
\|\nabla F_i(\theta)-\nabla F(\theta)\|^2.
\]
In particular, under Assumption~\ref{ass:heterogeneity_noniid}, $\mathbb{E}[\|\widehat{\nabla} F_I(\theta)-\nabla F(\theta)\|^2]\le \sigma_{\mathrm{priv}}^2+\zeta^2$.
\end{lemma}

\begin{proof}
Unbiasedness follows from $\mathbb{E}[\widehat{\nabla} F_I(\theta)]=K^{-1}\sum_{i=1}^{K}\mathbb{E}[\widehat{\nabla} F_i(\theta)]=\nabla F(\theta)$. Also,
\[
\mathbb{E}\|\widehat{\nabla} F_I-\nabla F\|^2
=
\frac{1}{K}\sum_{i=1}^{K}
\mathbb{E}\|\widehat{\nabla} F_i-\nabla F\|^2.
\]
Adding and subtracting $\nabla F_i$ gives $\widehat{\nabla} F_i-\nabla F=(\widehat{\nabla} F_i-\nabla F_i)+(\nabla F_i-\nabla F)$. The cross term vanishes because $\mathbb{E}[\widehat{\nabla} F_i-\nabla F_i]=0$, so $\mathbb{E}\|\widehat{\nabla} F_i-\nabla F\|^2\le\sigma_{\mathrm{priv}}^2+\|\nabla F_i-\nabla F\|^2$. Averaging over $i$ proves the claim.
\end{proof}

\subsection{Proof of Theorem~\ref{thm:stationarity_noniid}}
\label{app:proof_stationarity_noniid}

Let $t=0,\dots,T-1$ and let $q(t)$ be the communication round corresponding to step $t$. The SGD update is $\theta^{t+1}=\theta^t-\eta g^t$.

\paragraph{Bias--noise decomposition.}
Define the surrogate gradient bias
\[
b^t
:=
\nabla \Phi^{(q(t))}(\theta^t)-\nabla \bar\Phi(\theta^t)
=
\frac{1}{2}
\big(\nabla_\theta G_t(\theta^t)-\nabla_\theta G(\theta^t)\big).
\]
On the event of Lemma~\ref{lem:sur_ref_gap_noniid}, $\|b^t\|\le \Delta_T(\rho)/2$ for all $t$. Define $\xi^t:=g^t-\nabla \Phi^{(q(t))}(\theta^t)$. By Assumption~\ref{ass:smooth_sgd_noniid}(2), $\mathbb{E}[\xi^t\mid \theta^t,q(t)]=0$ and $\mathbb{E}[\|\xi^t\|^2\mid \theta^t,q(t)]\le\sigma^2+\zeta^2/4$. Hence
\[
g^t
=
\nabla \bar\Phi(\theta^t)+b^t+\xi^t.
\]

\paragraph{One-step descent on $\bar\Phi$.}
By $\Lambda$-smoothness of $\bar\Phi$ along iterates,
\[
\bar\Phi(\theta^{t+1})
\le
\bar\Phi(\theta^t)
-\eta\langle \nabla \bar\Phi(\theta^t),g^t\rangle
+\frac{\Lambda\eta^2}{2}\|g^t\|^2.
\]
Taking conditional expectation given $\theta^t,q(t)$, using Young's inequality, and bounding the squared norm gives
\[
\mathbb{E}[\bar\Phi(\theta^{t+1})\mid \theta^t,q(t)]
\le
\bar\Phi(\theta^t)
-\eta\left(\frac12-\Lambda\eta\right)\|\nabla \bar\Phi(\theta^t)\|^2
+
\left(\frac{\eta}{2}+\Lambda\eta^2\right)\|b^t\|^2
+
\frac{\Lambda\eta^2}{2}\left(\sigma^2+\frac14\zeta^2\right).
\]
Choose $\eta\le 1/(4\Lambda)$ so that $(1/2-\Lambda\eta)\ge 1/4$ and $(\eta/2+\Lambda\eta^2)\le 3\eta/4$. Since $\|b^t\|\le \Delta_T(\rho)/2$,
\[
\mathbb{E}[\bar\Phi(\theta^{t+1})]
\le
\mathbb{E}[\bar\Phi(\theta^t)]
-\frac{\eta}{4}\mathbb{E}[\|\nabla \bar\Phi(\theta^t)\|^2]
+
\frac{3\eta}{16}\Delta_T(\rho)^2
+
\frac{\Lambda\eta^2}{2}\left(\sigma^2+\frac14\zeta^2\right).
\]

\paragraph{Telescope.}
Summing over $t=0,\dots,T-1$ and using $\mathbb{E}[\bar\Phi(\theta^T)]\ge\bar\Phi^\star:=\inf_\theta\bar\Phi(\theta)$ gives
\[
\frac{1}{T}\sum_{t=0}^{T-1}\mathbb{E}[\|\nabla \bar\Phi(\theta^t)\|^2]
\le
\frac{4(\bar\Phi(\theta^0)-\bar\Phi^\star)}{\eta T}
+
\frac{3}{4}\Delta_T(\rho)^2
+
2\Lambda\eta\left(\sigma^2+\frac14\zeta^2\right),
\]
which proves Theorem~\ref{thm:stationarity_noniid}.
\qed

\paragraph{Proof overview and standard assumptions.}
Our convergence analysis follows a standard nonconvex SGD template. Lemma~\ref{lem:sur_ref_gap_noniid} bounds the gap between the empirical public gradient and the reference public gradient uniformly over $T$ steps by decomposing it into prompt covariate shift, pseudo-label conditional mismatch, and finite-sample concentration. Theorem~\ref{thm:stationarity_noniid} then plugs this uniform gap into a biased-SGD descent argument using $\Lambda$-smoothness, a bias--noise decomposition, Young's inequality, and telescoping. The assumptions are standard in nonconvex SGD and federated convergence analyses~\cite{ghadimi2013stochastic,karimireddy2020scaffold,li2020fedavg}; bounded-gradient assumptions are common and can be enforced via gradient clipping~\cite{li2020fedavg,abadi2016deep,chen2020understanding}; gradient-dissimilarity assumptions appear in FedProx/SCAFFOLD-style analyses~\cite{li2018fedprox,karimireddy2020scaffold}; KL/TV controls are standard in covariate-shift and domain-adaptation theory~\cite{sugiyama2007covariate,aminian2022information}; and modeling pseudo-labels as noisy labels or conditional mismatch aligns with noisy-label and self-training analyses~\cite{natarajan2013learning,wei2020theoretical}.

\section{Detailed Privacy Analysis of \fedcofit}
\label{app:privacy}

This appendix provides a detailed privacy analysis of DP-\fedcofit. We show that if each client's local LoRA fine-tuning procedure is differentially private, then the responses released on the public prompt set are also differentially private, and the full multi-round mechanism remains differentially private under adaptive composition.

\paragraph{Setup.}
Let $K$ be the number of clients, with client index set $[K]=\{1,\dots,K\}$. Training proceeds for $Q$ communication rounds, with round index set $[Q]=\{1,\dots,Q\}$. For each client $i\in[K]$, let $D_i=\{z_{i,1},\dots,z_{i,n_i}\}$ denote the private local dataset of client $i$, where $n_i=|D_i|$. Let $D_{\mathrm{pub}}=\{x_1,\dots,x_M\}$ denote the public prompt set shared by all clients, where $M=|D_{\mathrm{pub}}|$.

Each client maintains a local language model. For client $i$, $\phi_i$ denotes the trainable LoRA parameters, while the backbone parameters are frozen. At communication round $q\in[Q]$, the client updates its LoRA parameters using its private dataset $D_i$ and the information available from previous rounds.

\paragraph{Round history.}
For each round $q\in[Q]$, let $H^{(q-1)}$ denote the history available before round $q$. This history may include all previous client releases, server-side computations, and pseudo-labels or broadcasts sent back to clients up to round $q-1$. The exact contents of $H^{(q-1)}$ are not important for the privacy argument; the only requirement is that round $q$ may depend adaptively on earlier rounds through $H^{(q-1)}$.

\paragraph{Client-side private fine-tuning mechanism.}
At round $q$, client $i$ applies a randomized local training mechanism $\phi_i^{(q)}\leftarrow \mathcal{A}_{i,q}(D_i,H^{(q-1)})$. We assume that, for every fixed history $H^{(q-1)}$, the mechanism $\mathcal{A}_{i,q}(\cdot,H^{(q-1)})$ is $(\varepsilon_{i,q},\delta_{i,q})$-differentially private with respect to $D_i$.

\paragraph{Released responses on the public prompt set.}
After the local update, client $i$ evaluates its local model on each public prompt $x_j\in D_{\mathrm{pub}}$. We write $\hat{y}_{i,j}^{(q)}\leftarrow \mathcal{M}_i(x_j;\phi_i^{(q)})$ for $j\in[M]$. The round-$q$ message released by client $i$ is $\mathcal{R}_i^{(q)}=\{\hat{y}_{i,j}^{(q)}\}_{j=1}^{M}$.

\paragraph{Server-side processing.}
After receiving $\{\mathcal{R}_i^{(q)}\}_{i=1}^{K}$, the server performs the standard \fedcofit consensus step, including embedding the generated responses, clustering them in semantic space, selecting representative pseudo-labels, and broadcasting the pseudo-labeled public set back to clients. These operations are functions only of the released client messages and public information.

\begin{definition}[Neighboring datasets]
Fix a client $i\in[K]$. Two datasets $D_i$ and $D_i'$ are neighboring if they differ in exactly one training example.
\end{definition}

\begin{definition}[Differential privacy]
A randomized mechanism $\mathcal{A}$ is $(\varepsilon,\delta)$-differentially private with respect to client dataset $D_i$ if, for every pair of neighboring datasets $D_i,D_i'$ and every measurable set of outputs $S$,
\[
\Pr[\mathcal{A}(D_i)\in S]
\le
e^\varepsilon \Pr[\mathcal{A}(D_i')\in S]+\delta.
\]
\end{definition}

\paragraph{Round-level privacy.}
We first show that the public-prompt responses released by a client in a single round inherit the privacy guarantee of the local training mechanism.

\begin{lemma}[Privacy of released round-$q$ responses]
\label{lem:round_prediction_privacy}
Fix a client $i\in[K]$ and a communication round $q\in[Q]$. Suppose that, for the fixed history $H^{(q-1)}$, the local mechanism $\mathcal{A}_{i,q}(\cdot,H^{(q-1)})$ is $(\varepsilon_{i,q},\delta_{i,q})$-differentially private with respect to $D_i$. Then the released message $\mathcal{R}_i^{(q)}=\{\hat{y}_{i,j}^{(q)}\}_{j=1}^{M}$ is also $(\varepsilon_{i,q},\delta_{i,q})$-differentially private with respect to $D_i$.
\end{lemma}

\begin{proof}
For fixed history $H^{(q-1)}$, the output $\phi_i^{(q)}$ is produced by $\mathcal{A}_{i,q}(D_i,H^{(q-1)})$, which is $(\varepsilon_{i,q},\delta_{i,q})$-differentially private by assumption. The released message $\mathcal{R}_i^{(q)}$ is obtained by applying $\mathcal{M}_i(\cdot;\phi_i^{(q)})$ to the public prompt set $D_{\mathrm{pub}}$. Since $D_{\mathrm{pub}}$ is public and fixed independently of $D_i$, the map $\phi_i^{(q)}\mapsto \mathcal{R}_i^{(q)}$ is a function of the private mechanism output and public information only. Therefore, by post-processing, $\mathcal{R}_i^{(q)}$ is also $(\varepsilon_{i,q},\delta_{i,q})$-differentially private with respect to $D_i$.
\end{proof}

\paragraph{Multi-round privacy.}
Let $\mathcal{G}_i(D_i)=(\mathcal{R}_i^{(1)},\mathcal{R}_i^{(2)},\dots,\mathcal{R}_i^{(Q)})$ denote the overall multi-round mechanism induced by client $i$. Since each $\mathcal{R}_i^{(q)}$ may depend on the previous history $H^{(q-1)}$, $\mathcal{G}_i$ is an adaptive composition of the per-round release mechanisms.

\begin{theorem}[Record-level privacy of DP-\fedcofit]
\label{thm:appendix-dp-fedcofit}
Suppose that, for every communication round $q\in[Q]$ and every fixed history $H^{(q-1)}$, the client-side mechanism $\mathcal{A}_{i,q}(\cdot,H^{(q-1)})$ is $(\varepsilon_{i,q},\delta_{i,q})$-differentially private with respect to $D_i$. Then:
\begin{enumerate}
    \item the released public-prompt responses $\mathcal{R}_i^{(q)}$ are $(\varepsilon_{i,q},\delta_{i,q})$-differentially private with respect to $D_i$ for every round $q\in[Q]$;
    \item the overall multi-round mechanism $\mathcal{G}_i(D_i)=(\mathcal{R}_i^{(1)},\dots,\mathcal{R}_i^{(Q)})$ is
    \[
    \left(\sum_{q=1}^{Q}\varepsilon_{i,q},\ \sum_{q=1}^{Q}\delta_{i,q}\right)
    \text{-differentially private}
    \]
    with respect to $D_i$.
\end{enumerate}
\end{theorem}

\begin{proof}
Part (1) follows directly from Lemma~\ref{lem:round_prediction_privacy}. For Part (2), the sequence $\mathcal{R}_i^{(1)},\dots,\mathcal{R}_i^{(Q)}$ is adaptive because the mechanism used at round $q$ may depend on $H^{(q-1)}$. Differential privacy is closed under adaptive composition, so composing the per-round mechanisms over $Q$ rounds gives the stated guarantee. Finally, all server-side operations in \fedcofit, including response embedding, clustering, representative selection, and pseudo-label broadcast, are functions of released messages and public information only; hence, they incur no additional privacy loss by post-processing.
\end{proof}

\paragraph{Interpretation.}
The theorem establishes a record-level privacy guarantee for each client's private dataset. Sharing responses on the public prompt set does not invalidate privacy in DP-\fedcofit because those responses are outputs of a differentially private client-side LoRA fine-tuning mechanism.

\section{DETAILS ON EXPERIMENTS}
\label{app:exp:details}

\subsection{Environments, Datasets and Metrics}

\paragraph{Compute Resources}: All experiments are run on a server with an Intel(R) Platinum $8480$ and $5$ x NVIDIA H100 80GB.
\paragraph{Public Prompt Set}: Across all experiments, clients collaborate via a shared public prompt set $\mathcal{D}_p$, which is sampled from the corresponding dataset and excluded from the private fine-tuning data used by clients in that setting. 
\paragraph{Datasets}:
\paragraph{Dolly-15K.} An instruction-tuning dataset of $15K$ examples spanning diverse tasks (e.g., QA, summarization, extraction). In our instruction-compliance regime, we partition Dolly-15K across clients as private data and sample in-distribution public prompts from the same instruction pool when studying prompt-budget effects.
\paragraph{Alpaca.} A widely used instruction-following corpus (tens of thousands of instruction–response pairs). We use Alpaca as a private fine-tuning dataset in the instruction-compliance regime and also as an out-of-distribution source of public prompts for robustness tests.
\paragraph{Wizard.} A complex instruction-tuning dataset (Evol-Instruct style) emphasizing harder instructions and reasoning-heavy responses. We use Wizard for the knowledge/generalization regime as clients’ private fine-tuning data.
\paragraph{OpenOrca.} A GPT-augmented instruction dataset with reasoning traces (derived from FLAN-style tasks). We use OpenOrca as a private fine-tuning corpus in the knowledge/generalization regime and as an OOD public-prompt source in robustness experiments.
\paragraph{ShareGPT.} A chat-style conversation dataset collected from ShareGPT, containing multi-turn user–assistant interactions. We use ShareGPT as a private fine-tuning corpus for chat-quality experiments and also as a format-shifted public-prompt source; compared to instruction prompts, these chat-oriented prompts can induce more diverse generations and reduce consensus strength.
\paragraph{UltraChat.} A large-scale, high-quality conversational/instructional dialogue dataset. We use UltraChat as a private fine-tuning corpus for chat-quality experiments.

\paragraph{Benchmarks and Metrics:}
\paragraph{IFEval (Instruction-Following Evaluation).} We measure instruction-compliance by evaluating models fine-tuned on instruction data (Dolly-15K/Alpaca) using the IFEval benchmark, which tests whether a model follows explicit instruction constraints.
\paragraph{MMLU (Massive Multitask Language Understanding).} We evaluate knowledge and generalization after fine-tuning on Wizard/OpenOrca using MMLU, a broad multi-subject benchmark designed to probe factual knowledge and reasoning across many domains.
\paragraph{MT-Bench.} We evaluate chat quality after fine-tuning on ShareGPT/UltraChat using MT-Bench, a multi-turn open-ended benchmark for chat assistants (commonly judged by a strong LLM-as-a-judge protocol).

\section{Experiment setup}
\label{app:exp:setup}
We evaluate \fedcofit across three adaptation regimes: instruction compliance, knowledge, and generalization and chat quality. For instruction compliance, clients fine-tune on Dolly-15K~\cite{zhang2024towards} and Alpaca~\cite{taori2023stanford}, and we evaluate instruction-following behaviour on IFEval~\cite{zhou2023instruction}.
For knowledge and generalization, clients fine-tune on Wizard~\cite{luo2023wizardmath} and OpenOrca~\cite{OpenOrca}, and we evaluate on MMLU~\cite{hendrycks2020measuring}.
For chat quality, clients fine-tune on ShareGPT~\cite{chiang2023vicuna} and UltraChat~\cite{ding2023enhancing}, and we evaluate on MT-Bench~\cite{zheng2023judging}. We consider three foundation models: TinyLlama, LLaMA-7B, and LLaMA2-13B. For TinyLlama and LLaMA-7B we use $K=10$ clients, and for LLaMA2-13B, we use $K=5$ clients. All methods are evaluated under the same federation budget ( same number of clients and communication rounds ). We run IFEval experiments for $3$ communication rounds, MMLU experiments for $2$ rounds, and MT-Bench experiment for $1$ communication round. In \fedcofit experiments, clients share a public prompt set $D_{\mathrm{pub}}$ constructed by sampling prompts from the corresponding fine-tuning dataset and excluding them from private training splits. We use a fixed public prompt budget of $M=500$ prompts. We apply LoRA with rank $r=32$. For all models, we adapt the attention projection modules:\texttt{q\_proj}, \texttt{k\_proj}, \texttt{v\_proj}, \texttt{o\_proj}. We use a micro-batch size of $16$ and an effective batch size of $128$. The learning rate is set to $3\times 10^{-4}$ across experiments.

\section{Extra Experiment results}

\subsection{Effect of Public Prompt Budget:}
\label{main:eff:promp:size}
We first vary the number of public prompts by subsampling $M \in \{1000, 800, 500, 100, 50\}$ prompts from the Dolly-15K instruction pool, yielding public sets that are in-distribution with respect to the private fine-tuning data. Fig. \ref{fig:u_size} shows that \fedcofit is stable for moderate public prompt size as performance is nearly unchanged at $M=1000,800,500$ (IFEval$\approx 0.514-0.510$). When the size is reduced to $M \le 100$, performance drops to $\approx 0.481$. Importantly, even at large budgets, \fedcofit matches or exceeds strong baselines (e.g., FLoRA at $0.5026$) while at very small budgets it remains substantially above FedIT ($0.4382$).

\begin{figure}[H]
\begin{minipage}[H]{0.5\textwidth}
    \centering
    \includegraphics[width=\linewidth]{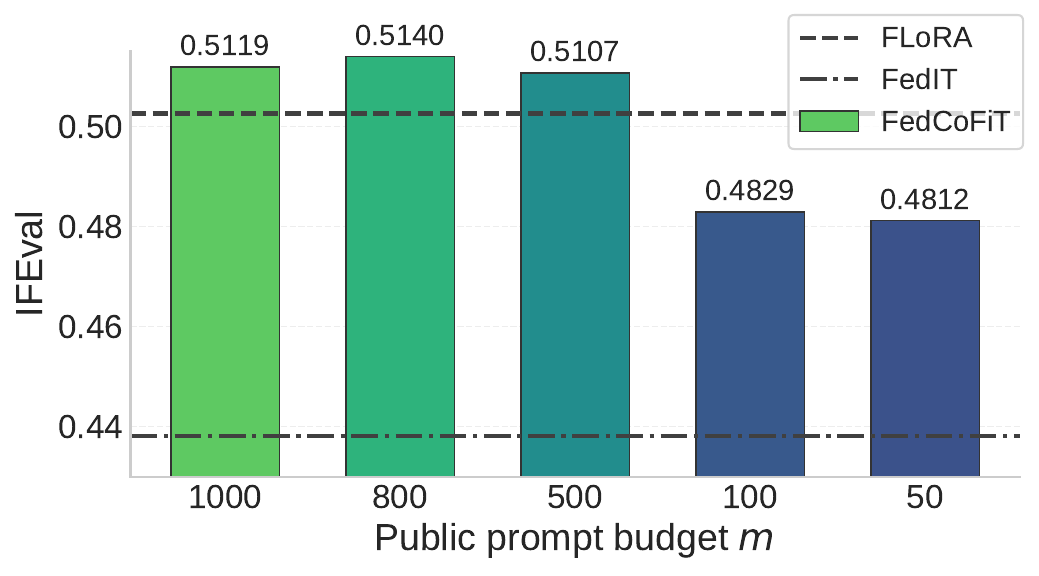}
    \caption{IFEval vs public prompt budget $M$.}
    \label{fig:u_size}
\end{minipage}\hfill
\begin{minipage}[H]{0.5\textwidth}
    \centering
    \includegraphics[width=\linewidth]{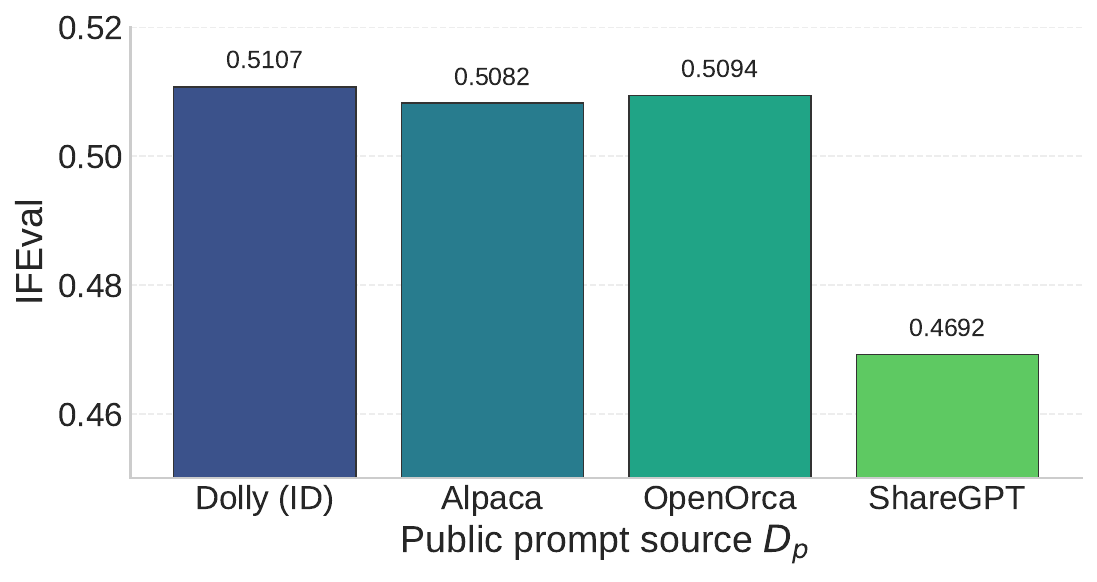}
    \caption{\fedcofit robustness to public-prompt distribution shift}
    \label{fig:dis_shift}
\end{minipage}
\end{figure}

\subsection{Effect of public prompt distribution shift:}
\label{app:prompt_distribution_shift}
Next, we test whether \fedcofit requires the public prompts to match the client data distribution. We fix the prompt budget to a constant $500$ and keep the private fine-tuning data unchanged (Dolly-15k), but replace the public prompt source. Specifically, we compare (1) $D_{\mathrm{pub}}$ sampled from Dolly-15k prompts (in-distribution) against out-of-distribution public prompts sampled from (2) Alpaca, (3) OpenOrca, and (4) ShareGPT, using prompts only in all cases. We report IFEval score for each setting. As shown in Fig. \ref{fig:dis_shift}, using instruction-style public prompts such as Alpaca and OpenOrca yields only minor changes compared to Dolly-15k. This suggests that \fedcofit is robust to moderate shifts in the public prompt distribution as long as the prompts remain instruction-following and aligned in format. In contrast, using ShareGPT prompts leads to a drop to $0.4692$. We attribute this to a stronger distribution shift. ShareGPT prompts are chat-oriented and often reflect multi-turn conversational context, informal style, and different intent distributions than Dolly-style single-run instructions. As a result, client generations on these prompts become more diverse in surface form and content. Overall, these results indicate that \fedcofit does not require a perfectly matched public prompt set, but benefits from public prompts that are at least format-consistent with the downstream adaptation task.

\begin{figure*}
    \centering
    \includegraphics[width=0.9\linewidth]{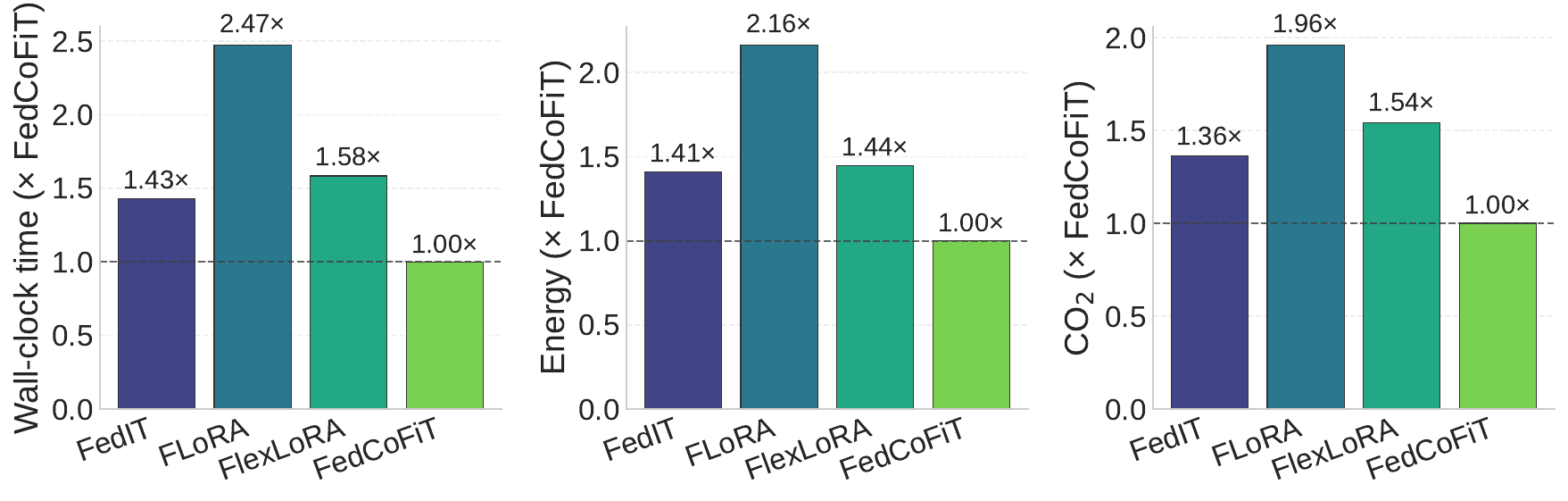}
    \caption{Efficiency comparison on Dolly-15k using Llama-7B}
    \label{fig:eff:llama:main}
\end{figure*}

\subsection{\fedcofit Efficiency Evaluation}
\label{app:eff:eval}
Beyond accuracy and communication volume, we evaluate the efficiency of \fedcofit in terms of end-to-end wall clock runtime, energy consumption, and estimated $\mathrm{CO}_2$ emissions. We measure runtime, aggregate energy in kwh and $\mathrm{CO}_2$ in kg using the Lamarr Energy tracker ~\cite{ai_energy_validation} under a realistic bandwidth. We report end-to-end measurements that include local fine-tuning, public-prompt generation, server-side consensus construction, and communication overhead. To ensure a fair comparison, all methods are run with the same number of clients, the same number of communication rounds, and matched local training budgets. Hardware and hyperparameter details are provided in Appendix~\ref{app:exp:details}. Figure \ref{eff-llama} shows that across the main experimental settings, \fedcofit achieves comparable task performance to strong aggregation baselines while reducing end-to-end wall-clock time and energy consumption, and consequently lowering estimated $\mathrm{CO}_2$ emissions. The gains are most pronounced in regimes where communication and synchronization dominate the runtime (e.g., larger backbones and higher-rank LoRA baselines), consistent with our communication-scaling analysis. 

\begin{figure*}
    \centering
    \includegraphics[width=0.9\linewidth]{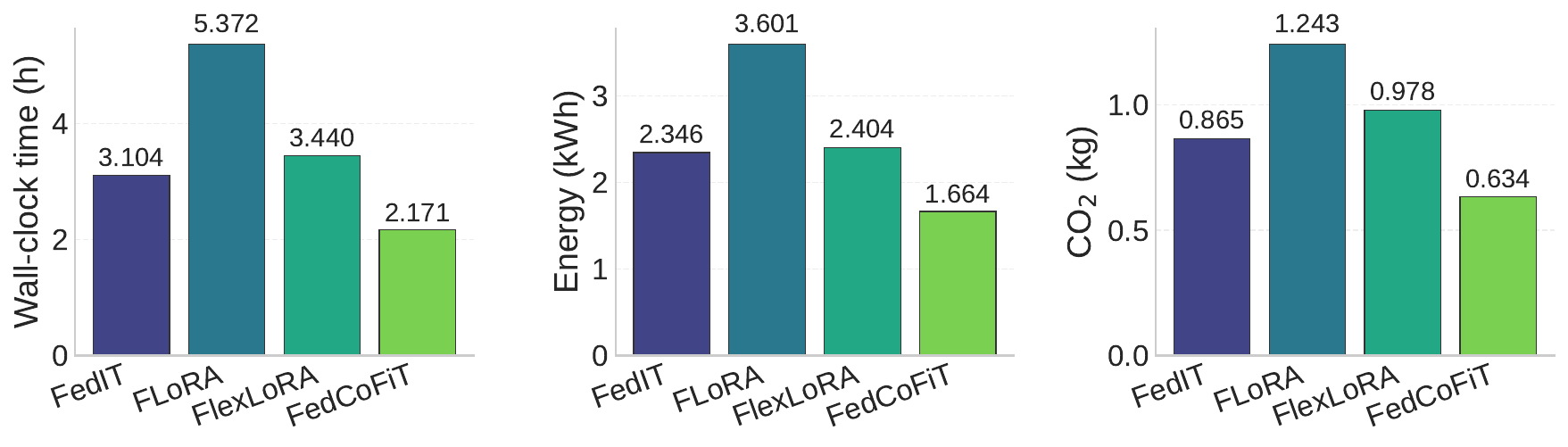}
    \caption{Efficiency comparison on Dolly-15 using Llama-7B}
    \label{eff-llama}
\end{figure*}

\subsection{Choice of Semantic Encoder}
\label{app:encoder_ablation}
\fedcofit builds server-side consensus in a semantic embedding space, where each client-generated response is mapped by a sentence encoder into a fixed-dimensional vector. The encoder choice affects both (1) the quality of semantic similarity and thus the quality of pseudo-labels and (2) algorithm efficiency, since embeddings are computed for every public prompt and client output each round. We therefore ablate a small set of widely used SBERT-style encoders that span different speed/size regimes. We fix the federated setup (TinyLlama clients, Dolly-15K private fine-tuning, fixed public prompt budget, identical decoding and clustering hyperparameters) and vary only the semantic encoder. We then evaluate the resulting fine-tuning model on IFEval benchmark. Table \ref{tab:encoder_ablation} summarizes encoder size, embedding dimension, encoding throughput (sentence/sec), and the downstream IFEval score obtained when fine-tuning on Dolly-15K and evaluating on IFEval. Across the tested encoders, all-MiniLM-L6-v2 provides the best overall trade-off. While larger encoders such as all-mpnet-base-v2 can be competitive, they are substantially slower and heavier. Multilingual encoders (e.g., distiluse-base-multilingual-cased-v1) also introduce additional overhead without improving Dolly/IFEval performance in our instruction-following setting. We therefore select all-MiniLM-L6-v2 as the default semantic encoder in all experiments due to its strong IFEval performance coupled with low model size and high encoding speed.
\begin{table}[H]
\centering
\small
\setlength{\tabcolsep}{5pt}
\renewcommand{\arraystretch}{1.05}
\caption{Semantic encoder ablation for FedCoFiT. We report encoder size, embedding dimension, encoding throughput (sentences/sec), and downstream IFEval when training on Dolly-15K and evaluating on IFEval.}
\label{tab:encoder_ablation}
\begin{tabular}{lcccc}
\toprule
\textbf{Encoder} & \textbf{Size (MB)} & \textbf{Dim} & \textbf{Speed (sent/s)} & \textbf{IFEval} $\uparrow$ \\
\midrule
all-mpnet-base-v2                    & 420 & 768 & 2800  & 52.47 \\
all-distilroberta-v1                 & 290 & 768 & 4000  & 51.29 \\
all-MiniLM-L12-v2                    & 120 & 384 & 7500  & 51.12 \\
\textbf{all-MiniLM-L6-v2}            & \textbf{80}  & \textbf{384} & \textbf{14200} & \textbf{51.19} \\
distiluse-base-multilingual-cased-v1 & 480 & 512 & 4000  & 52.17 \\
\bottomrule
\end{tabular}
\end{table}

\subsection{Ablation on Consensus Representative Selection}
\label{app:representative_selection_ablation}

\fedcofit selects one client-generated response as the pseudo-label for each public prompt. 
To isolate the effect of this selection rule, we run \fedcofit on TinyLlama with Dolly-15K private fine-tuning data and evaluate the resulting models on IFEval. 
We keep the semantic encoder, DBSCAN clustering hyperparameters, decoding configuration, public prompt budget $M=500$, and local training setup fixed, and vary only how the final pseudo-label $y_j^\star$ is chosen. 
We compare three alternatives: (i) our default \emph{majority-cluster centroid} rule, which first selects the largest non-outlier cluster and then chooses the response closest to its centroid; (ii) \emph{majority-cluster random}, which selects a uniformly random response from the largest non-outlier cluster; and (iii) \emph{global medoid}, which ignores clustering and chooses the response with the smallest average distance to all other client responses. 
This ablation tests whether both parts of the consensus rule are useful: majority-cluster selection filters minority modes and outliers, while centroid-based representative selection avoids arbitrary pseudo-label choice within the consensus cluster. Table~\ref{tab:representative_selection_ablation} shows that our default representative-selection rule performs best. Selecting the response closest to the centroid of the majority cluster improves over choosing a random response from the same cluster, indicating that centrality within the consensus cluster matters for pseudo-label quality.  The larger drop for the global-medoid baseline shows that clustering and outlier filtering are also important: choosing the most central response over all client outputs can be affected by minority modes or noisy generations.  These results support the two-stage design of \fedcofit, where the server first identifies the dominant semantic mode and then selects the most representative response within that mode.

\begin{table}[H]
\centering
\small
\setlength{\tabcolsep}{5pt}
\renewcommand{\arraystretch}{1.05}
\caption{Ablation of pseudo-label representative selection for \fedcofit on TinyLlama with Dolly-15K private data and IFEval evaluation. We keep DBSCAN, the semantic encoder, decoding, public prompt budget, and training setup fixed, and vary only how the final pseudo-label $y_j^\star$ is selected.}
\label{tab:representative_selection_ablation}
\begin{tabular}{lccc}
\toprule
\textbf{Selection rule} & \textbf{Uses majority cluster?} & \textbf{Outlier filtering?} & \textbf{IFEval} $\uparrow$ \\
\midrule
Majority-cluster centroid representative 
& Yes & Yes & \textbf{51.10} \\
Random representative from majority cluster 
& Yes & Yes & 50.06 \\
Global medoid over all responses 
& No & No & 48.17 \\
\bottomrule
\end{tabular}
\end{table}

\subsection{Communication Example: LLaMA-2-13B scale.}
\label{app:llama2:example}
Consider $K=10$ clients, $M=1024$ prompts per round, average response length $\bar{\ell}=128$, and $c_t\approx 2$ bytes/token. Then $C_{\mathrm{SC}}\approx 2KM\bar{\ell}c_t=2\cdot10\cdot1024\cdot128\cdot2\approx 5.2~\mathrm{MiB}$, where the leading factor $2$ accounts for upload and download.
For a LLaMA-2-13B-scale model with $L_{\mathrm{layers}}=40$, hidden size $d_{\mathrm{model}}=5120$, FP16 communication $c_{\mathrm{param}}=2$, and rank $r=32$, LoRA on $W_q$ and $W_v$ gives upload $C^{\uparrow}_{\mathrm{LoRA}(q,v)}=c_{\mathrm{param}}L_{\mathrm{layers}}2r(d_{\mathrm{model}}+d_{\mathrm{model}})=52{,}428{,}800$ bytes $\approx50.0~\mathrm{MiB}$, or about $100.0~\mathrm{MiB}$ upload plus download. If LoRA is applied to all attention projections and the MLP with $d_{\mathrm{ff}}=13824$, the upload is $C^{\uparrow}_{\mathrm{LoRA}(\mathrm{attn+MLP})}=c_{\mathrm{param}}L_{\mathrm{layers}}r\{4(d_{\mathrm{model}}+d_{\mathrm{model}})+3(d_{\mathrm{model}}+d_{\mathrm{ff}})\}\approx238.7~\mathrm{MiB}$, or about $477.4~\mathrm{MiB}$ total. Orthogonal communication-saving mechanisms, such as dynamic or event-triggered communication protocols, could in principle be combined with Semantic Consensus by skipping consensus rounds when local model behavior or its semantic embedding changes only marginally across rounds~\citep{kamp2019black,kamp2018efficient,kamp2016communication}.

\section{Impact Statement}
\label{app:impact}
\fedcofit has the potential to make federated fine-tuning of large language models substantially more practical by replacing costly parameter aggregation with semantic consensus over model outputs. By exchanging only generated responses to shared public prompts, the method reduces communication independently of model size, supports heterogeneous clients with different architectures, tokenizers, or LoRA ranks, and can operate even when only black-box model access is available. Empirically, the paper shows that FedCoFiT matches or exceeds strong federated LoRA baselines across instruction following, knowledge generalization, and chat-quality tasks while using orders of magnitude less communication, with additional reductions in runtime, energy use, and estimated emissions. Its compatibility with differential privacy further strengthens its relevance for sensitive domains such as healthcare, finance, and enterprise settings, where data cannot be centralized but collaborative adaptation is valuable.

\newpage

\end{document}